\documentclass{article}

\usepackage{microtype}
\usepackage{graphicx}
\usepackage{subfigure}
\usepackage{booktabs} 

\usepackage{amsmath,amsthm,amsfonts,amssymb}
\usepackage{color}
\usepackage{mathrsfs}
\usepackage{enumitem}
\usepackage{bm}
\usepackage{multirow}
\usepackage{booktabs}
\usepackage{makecell}
\usepackage{graphicx}
\usepackage{subfigure}
\usepackage{caption}
\usepackage{thmtools}
\usepackage{thm-restate}
\usepackage{hhline}
\usepackage{cite}
\usepackage[table]{xcolor}
\definecolor{light-gray}{gray}{0.9}
\usepackage{times}
\usepackage{rotating} 
\usepackage{natbib}
\usepackage{graphicx}
\usepackage{bm}
\usepackage{diagbox}
\usepackage{pdflscape}

\usepackage{hyperref}

\usepackage[accepted]{icml2020}

\icmltitlerunning{From Importance Sampling to Doubly Robust Policy Gradient}

\begin{document}

\newcommand{\defeq}{\mathrel{\mathop:}=}
\newcommand{\veps}{\varepsilon}
\renewcommand{\epsilon}{\varepsilon}

\newcommand{\trans}{^{\top}}

\newcommand{\cA}{\mathcal{A}}
\newcommand{\cS}{\mathcal{S}}
\newcommand{\Scal}{\mathcal{S}}
\newcommand{\Acal}{\mathcal{A}}
\newcommand{\EE}{\mathbb{E}}
\newcommand{\VV}{\mathbb{V}}

\newcommand{\btheta}{{\bm{\theta}}}
\newcommand{\pa}{{\pi_{\btheta+\Delta \btheta}}}
\newcommand{\pb}{{\pi_{\btheta}}}

\newcommand{\pta}[1]{{\pi_{\btheta+\Delta \btheta}^{#1}}}
\newcommand{\ptb}[1]{{\pi_{\boldsymbol{\theta}}^{#1}}}

\newcommand{\dt}{{\Delta\btheta}}
\newcommand{\bt}{{\bm\theta}}
\newcommand{\dpt}{{\Delta\pb}}
\newcommand{\hv}{\hat{V}}
\newcommand{\hq}{\hat{Q}}
\newcommand{\tv}{\tilde{V}}
\newcommand{\tq}{\tilde{Q}}
\newcommand{\tR}{\tilde{R}}
\newcommand{\tp}{\tilde{p}}
\newcommand{\hd}{\widehat{\text{DR}}}
\newcommand{\tpz}{\lim_{\dt\rightarrow \bm{0}}}
\newcommand{\gbt}{\nabla_{\btheta}}

\newcommand{\cov}{{\rm Cov}}

\newcommand\numberthis{\addtocounter{equation}{1}\tag{\theequation}}
\newcommand{\me}{{\mathbb{E}}}
\newcommand{\mv}{{\mathbb{V}}}
\newcommand{\trp}[1]{\Big(#1\Big)^2}

\newcommand{\si}{s_{\text{init}}}

\let\hat\widehat
\let\tilde\widetilde
\let\check\widecheck
\def\given{{\,|\,}}
\def\biggiven{{\,\big|\,}}
\def\Biggiven{{\,\Big|\,}}
\def\bigggiven{{\,\bigg|\,}}
\def\Bigggiven{{\,\Bigg|\,}}
\def\ds{\displaystyle}
\newcommand\wtilde{\stackrel{\sim}{\smash{\mathcal{W}}\rule{0pt}{1.1ex}}}

\newtheorem{theorem}{Theorem}[]
\newtheorem{lemma}[theorem]{Lemma}
\newtheorem{corollary}[theorem]{Corollary}
\newtheorem{remark}[theorem]{Remark}
\newtheorem{fact}[theorem]{Fact}
\newtheorem{property}[theorem]{Property}
\newtheorem{claim}{Remark}
\newtheorem{proposition}[theorem]{Proposition}
\theoremstyle{definition}
\newtheorem{definition}[theorem]{Definition}
\newtheorem{condition}[theorem]{Condition}
\newtheorem{example}[theorem]{Example}
\newtheorem{assumption}{Assumption}
\renewcommand*{\theassumption}{\Alph{assumption}}

\def\##1\#{\begin{align}#1\end{align}}
\def\$#1\${\begin{align*}#1\end{align*}}

\newcommand{\blue}[1]{{\color{blue} #1}}

\newcommand{\pt}[2]{\frac{\partial #1}{\partial \theta_{#2}}}

\newcommand{\dti}[1]{{\epsilon_{#1}}}

\newcommand{\dpi}[2]{{\Delta\pi^{#1}_{\theta_{#2}}}}

\newcommand{\turb}[1]{\btheta+\dti{#1}\bm{e}_{#1}}
\newcommand{\pai}[1]{{\pi_{\turb{#1}}}}

\newcommand{\tpzi}[1]{\lim_{\dti{#1}\rightarrow 0}}

\newcommand{\vect}[1]{\ensuremath{\mathbf{#1}}}

\newcommand{\ptai}[2]{{\pi_{\btheta+\dti{#2}\bm{e}_{#2}}^{#1}}}

\newcommand{\jpi}[1]{\hat{J}(#1)}

\newcommand{\diag}{{\rm Diag}}

\newcommand{\orange}[1]{{\color{orange} #1}}

\twocolumn[
\icmltitle{From Importance Sampling to Doubly Robust Policy Gradient}





\begin{icmlauthorlist}
\icmlauthor{Jiawei Huang}{UIUC}
\icmlauthor{Nan Jiang}{UIUC}
\end{icmlauthorlist}

\icmlaffiliation{UIUC}{Department of Computer Science, University of Illinois Urbana-Champaign}

\icmlcorrespondingauthor{Nan Jiang}{nanjiang@illinois.edu}

\icmlkeywords{reinforcement learning, off-policy evaulation, policy gradient}

\vskip 0.3in
]



\printAffiliationsAndNotice{}  


\begin{abstract}
We show that on-policy policy gradient (PG) and its variance reduction variants can be derived by taking finite difference of function evaluations supplied by estimators from the importance sampling (IS) family for off-policy evaluation (OPE). Starting from the doubly robust (DR) estimator \citep{jiang2016doubly}, we provide a simple derivation of a very general and flexible form of PG, which subsumes the state-of-the-art variance reduction technique  \citep{cheng2019trajectory} as its special case and immediately hints at further variance reduction opportunities overlooked by existing literature. We analyze the variance of the new DR-PG estimator, compare it to existing methods as well as the Cramer-Rao lower bound of policy gradient, and empirically show its effectiveness. 
\end{abstract}

\section{Introduction}
In reinforcement learning, policy gradient (PG) refers to the family of algorithms that estimate the gradient of the expected return w.r.t.~the policy parameters, often from on-policy Monte-Carlo trajectories. Off-policy evaluation (OPE) refers to the problem of evaluating a policy that is different from the data generating policy, often by \emph{importance sampling} (IS) techniques. 

Despite the superficial difference that standard PG is on-policy while IS for OPE is off-policy by definition, they share many similarities: both PG and IS are arguably based on the Monte-Carlo principle (as opposed to the dynamic programming principle); both of them often suffer from high variance, and variance reduction techniques have been studied extensively for PG and IS separately in the literature. Given these similarities, one may naturally wonder: \emph{is there a deeper connection between the two topics?}

\paragraph{Summary of the Paper} We provide a simple and positive answer to the above question in the episodic RL setting. In particular, one can write down the policy gradient as (we informally illustrate the idea with scalar $\theta$ for now)
\begin{align} \label{eq:grad}
\lim_{\Delta\theta\to 0} \frac{J(\pi_{\theta + \Delta\theta}) - J(\pi_{\theta})}{\Delta\theta},
\end{align}
where $\theta$ is the current policy parameter, and $J(\cdot)$ is the expected return of a policy w.r.t.~some initial state distribution. The connection between IS and PG is extremely simple: using any method in the IS family to estimate $J(\cdot)$ in Eq.\eqref{eq:grad} will lead to a version of PG, and most unbiased PG estimators---with different variance reduction techniques---can be recovered in this way. Furthermore, by deriving PG from the doubly robust (DR) estimator for OPE \citep{jiang2016doubly}, we obtain a very general and flexible form of PG with variance reduction, which immediately subsumes the state-of-the-art technique by \citet{cheng2019trajectory} as its special case. In fact, the resulting estimator can achieve more variance reduction than \citet{cheng2019trajectory} given additional side information. See Table~\ref{tab:big} for some highlighted results. 

\section{Related Work}

To the best of our knowledge,  \citet{jie2010connection} was the first to explicitly mention the connection between (per-trajectory) IS and PG, which corresponds to the first row of our Table~\ref{tab:big}. 
The connection between DR and PG was lightly touched by \citet{tucker2018mirage}, although the authors' main goal was to challenge the success of state-action-dependent baseline methods in benchmarks, and did not give a more detailed analysis on this connection.

More recently, \citet{cheng2019trajectory} noticed that the previous variance reduction methods in PG overlooked the correlation across the times steps and ignored the randomness in the future steps \citep[e.g.,][]{gu2017q, liu2018actiondependent, grathwohl2018backpropagation, wu2018variance}. They used the law of the total variance to derive a trajectory-wise control variate estimator, which is subsumed by our general form of PG derived from DR in  Section~\ref{sec:dr-pg} as a special case.

\onecolumn
\begin{landscape}
\begin{table}

\renewcommand\arraystretch{3}
\caption{\label{tab:big} OPE estimators and their corresponding PG estimators. Time index in the subscript often specifies the omitted function arguments, e.g., $\tv^{\pi'}_t := \tv^{\pi'}(s_t)$; see Section~\ref{sec:notation} for details. Also note that $b\equiv V^\pi$ considered in the variance column is for simplicity and is not the optimal baseline \cite{jie2010connection}.}

\newcommand{\tabincell}[2]{
\begingroup
\renewcommand*{\arraystretch}{1.3}
\begin{tabular}{@{}#1@{}}#2\end{tabular}
\endgroup
}

\centering
\renewcommand*{\arraystretch}{1.8}
\begin{tabular}{c|l|c|l|c}
\toprule[1.0pt]
~ & \multicolumn{1}{c|}{OPE} & Finite Diff $\Rightarrow$ & \multicolumn{1}{c|}{PG} & \tabincell{c}{(Co)Variance of PG in Deterministic MDPs,\\ with $b \equiv V^\pi$ and $\tilde Q^{(\cdot)} \equiv Q^{(\cdot)}$}\\
\hline
Traj-IS & $\displaystyle \rho_{[0:T]} \sum_{t=0}^T\gamma^t r_t$ & \tabincell{c}{\blue{Proposition~\ref{prop:traj_ope2pg}} \\ (\citeauthor{jie2010connection})} & $\displaystyle \sum_{t=0}^T\nabla \log \ptb{t}\sum_{t'=0}^T\gamma^{t'}r_{t'}$ &
Omitted (worse than step-IS) \\
\hline
Step-IS & $\displaystyle \sum_{t=0}^T\gamma^t\rho_{[0:t]} r_t$ &  \blue{Proposition~\ref{prop:stepIS_ope2pg}} & $\displaystyle\sum_{t=0}^T \nabla \log\ptb{t} \sum_{t'=t}^T \gamma^{t'} r_{t'}$ & $\displaystyle \me[\sum_{t=0}^{T}\gamma^{2t} \cov_{t}\big[\nabla Q^\pb_{t}+Q^\pb_{t}\sum_{t'=0}^t\nabla\log\ptb{t'}|s_{t}\big]]$ \\
\hline
Baseline & \tabincell{l}{$\displaystyle b_0 + \sum_{t=0}^T \gamma^{t}\rho_{[0:t]}\Big( r_{t} + \gamma b_{t+1} -b_t\Big)$} & \blue{Proposition~\ref{prop:baseline_pg2ope}} & $\displaystyle \sum_{t=0}^T \nabla \log \ptb{t} \Big(\sum_{t'=t}^{T}\gamma^{t'} r_{t'}-\gamma^t b_t\Big)$ & $\displaystyle \me[\sum_{t=0}^{T}\gamma^{2t} \cov_{t}\big[\nabla Q^\pb_{t}+A^\pb_{t}\sum_{t'=0}^t\nabla\log\ptb{t'}|s_{t}\big]]$ \\
\hline
\multirow{8}*{\tabincell{c}{Doubly\\ Robust}} & \tabincell{l}{Recursive Version \\ \cite{jiang2016doubly}} &~& \tabincell{l}{$\tq$ does not change with $\btheta$ \\ \cite{cheng2019trajectory}} \\
~ & \tabincell{l}{$\displaystyle \hd^{\pi'}_{t} := \tv^{\pi'}_t$ \\ \hspace{1em} $+\rho_t\Big(r_t + \gamma \hd^{\pi'}_{t+1}-\tq^{\pi'}_t\Big)$} &~& $\displaystyle \sum_{t=0}^T\Big\{\nabla\log\ptb{t}\Big[\sum_{t'=t}^T\gamma^{t'}r_{t'}$ & $\displaystyle \me[\sum_{t=0}^{T}\gamma^{2t} \cov_{t}\big[\nabla Q^\pb_{t}|s_{t}]]$\\ 
~ & Expanded Version &~& $\displaystyle +\sum_{t'=t+1}^T\gamma^{t'}\Big(\tv^\pb_{t'}-\tq^\pb_{t'}\Big)\Big]$ & \\

~ &  $\displaystyle \tv^{\pi'}_0 + \sum_{t=0}^T\gamma^t \rho_{[0:t]}\Big(r_t +  \gamma \tv^{\pi'}_{t+1}-\tq^{\pi'}_t\Big)$ &  \blue{Theorem~\ref{thm:dr_ope2pg}} & $\displaystyle+\gamma^t\Big(\nabla\tv^\pb_t-\tq^\pb_t\nabla\log\ptb{t}\Big)\Big\}$ &  \\ 
\cline{4-5}
~ &  & ~& $\tq$ is a function of $\btheta$ \blue{(new)} \\
~ &  & ~& $\displaystyle \sum_{t=0}^T\Big\{\nabla\log\ptb{t}\Big[\sum_{t'=t}^T\gamma^{t'}r_{t'}$ & $\mathbf{0}$ (zero matrix) \\
~ & ~ & ~& $\displaystyle +\sum_{t'=t+1}^T\gamma^{t'}\Big(\tv^\pb_{t'}-\tq^\pb_{t'}\Big)\Big]$ & \\
~& ~ &~& $\displaystyle +\gamma^t\Big({\color{blue}\nabla\tv^\pb_t-\nabla\tq^\pb_t}$\\
~&~&~&$\displaystyle-\tq^\pb_t\nabla\log\ptb{t}\Big)\Big\}$ &\\

\hline
\tabincell{c}{Actor \\ Critic} & $\displaystyle  \sum_{t=0}^T\gamma^t\rho_{[0:t]}\Big(f_t-\gamma f_{t+1}\Big)$ &  \blue{Proposition~\ref{prop:ac_pg2ope}} & $\displaystyle\sum_{t=0}^T \gamma^t \nabla\log \ptb{t} \cdot  f_t$ &
Omitted (biased estimator) \\
\hline

\toprule[1.0pt]
\end{tabular}
\end{table}
\end{landscape}

\twocolumn

\section{Preliminaries}
\subsection{Markov Decision Processes (MDPs)} \label{sec:mdp}
We consider episodic RL problems with a fixed horizon, formulated as an  MDP $M = (\Scal, \Acal, P, R, T, \gamma, s_0)$, where $\Scal$ is the state space and $\Acal$ is the action space. For the ease of exposition we assume both $\Scal$ and $\Acal$ are finite and discrete.\footnote{Note that both PG and IS occur no explicit dependence on $|\Scal|$ or $|\Acal|$, and estimators derived for the discrete case can be extended to continuous state and action spaces.} $P: \Scal\times\Acal \to \Delta(\Scal)$ is the transition function, $R: \Scal\times\Acal \to \Delta(\mathbb{R})$ is the reward function, and $T$ is the horizon (or episode length). It is optional but we also include a discount factor $\gamma \in [0, 1)$ for more flexibility, which will later allow us to express the estimators in the IS and the PG literature in consistent notations. $s_0$ is the deterministic start state, which is without loss of generality. We will also assume that state contains the time step information (so that value functions are stationary); in other words, each state can only appear at a particular time step. Overall, these assumptions are only made for notational simplicity, and do not limit the generality of our derivations. 

A (stochastic) policy $\pi: \Scal\to\Delta(\Acal)$ induces a random trajectory $
s_0, a_0, r_0, s_1, a_2, r_2, s_3, \ldots, s_T, a_T, r_T
$, 
where $a_t \sim \pi(s_t)$, $r_t \sim R(s_t, a_t)$, and $s_{t+1} \sim P(s_t, a_t)$ for all $0 \le t \le T$. The ultimate measure of the performance of $\pi$ is the expected return, defined as
$$ \textstyle
J(\pi) := \EE\left[\sum_{t=0}^T \gamma^{t} r_t \,|\,a_{0:T} \sim \pi\right],
$$
where $a_{0:T}$ is the shorthand for $a_t \sim \pi(s_t)$ for $0\le t\le T$. 
It will be useful to define the state-value and Q-value functions:
for $s$ that may appear in time step $t$ (recall that we assume $t$ is encoded in $s$), 
\begin{align*}\textstyle
& V^\pi(s) := \EE\left[\sum_{t'=t}^\infty \gamma^{t'-t} r_{t'} \,|\, s_t = s, a_{t:T} \sim \pi\right], \\
& Q^\pi(s, a) := \EE\left[\sum_{t'=t}^\infty \gamma^{t'-t} r_{t'} \,|\, s_t = s, a_t = a, a_{t+1:T} \sim \pi \right].
\end{align*}
For simplicity we treat $s_{T+1}$ as a special terminal (absorbing) state, such that any (approximate or estimated) value function always evaluates to $0$ on $s_{T+1}$. 

\subsection{Off-Policy Evaluation and Importance Sampling}
Off-policy evaluation (OPE) is the problem of estimating the expected return of a policy $\pi'$ from data collected using a different policy $\pi$. Importance sampling (IS) is a standard technique for OPE. Given a trajectory $
s_0, a_0, r_0, s_1, a_2, r_2, s_3, \ldots, s_T, a_T, r_T
$ where all actions are taken according to $\pi$, (step-wise) IS forms the following unbiased estimate of $J(\pi')$ \citep{precup2000eligibility}:
\begin{align}
\jpi{\pi'} = \sum_{t=0}^T \gamma^t \prod_{t'=0}^t \frac{\pi'(a_{t'}|s_{t'})}{\pi(a_{t'}|s_{t'})} r_t. 
\end{align}
The estimator for a dataset of multiple trajectories will be simply the average of the above estimator applied to each trajectory. Since such a pattern is found in all estimators we consider (including the PG estimators), we will always consider only a single trajectory in the analyses.

The term $\frac{\pi'(a_{t}|s_{t})}{\pi(a_{t}|s_{t})}$ is often called the importance weight/ratio. We will use $\rho_t$ as its shorthand, and $\rho_{[t_1:t_2]}$ is the shorthand for its cumulative product, $\prod_{t'=t_1}^{t_2} \rho_{t'}$, with $\rho_{[t_1:t_2]} := 1$ when $t_1 > t_2$. With the above shorthand, the step-wise IS estimator can be succinctly expressed as 
\begin{align} \label{eq:step_is}
\jpi{\pi'} = \sum_{t=0}^T \gamma^t \rho_{[0:t]} r_t.
\end{align} 
We will be referring to multiple OPE estimators throughout the paper. Instead of giving each estimator a separate variable name, we will just use a generic notation $\jpi{\cdot}$, and the specific estimator it refers to should be clear from the surrounding text and theorem statements.

\paragraph{Doubly Robust (DR) Estimator \citep{jiang2016doubly, thomas2016data}}~\\
The DR estimator uses an approximate value function $\tq^{\pi'}$ to reduce the variance of IS via control variates. In its expanded form, the estimator is
\begin{align} \label{eq:dr}
\hat{J}(\pi') = \tv^{\pi'}(s_0) + \sum_{t=0}^T\gamma^t \rho_{[0:t]}\Big(r_t + \gamma \tv^{\pi'}(s_{t+1}) \nonumber \\ -\tq^{\pi'}(s_t, a_t)\Big),
\end{align}
where $\tv^{\pi'}(s) := \EE_{a \sim \pi'(s)}[\tq^{\pi'}(s, a)].$ \citet{jiang2016doubly} showed that DR has maximally reduced variance, in the sense that when $\tq^{\pi'}$ is accurate, there exists RL problems (typically tree-MDPs) where the variance of the estimator is equal to the Cramer-Rao lower bound of the estimation problem. As we will see later in Section~\ref{sec:dr-pg}, the PG estimator induced by DR also achieves the state-of-the-art variance reduction, and the variance when both $\tilde Q$ and $\nabla_\btheta \tilde Q$ are accurate also coincides with the C-R bound for PG. 

\subsection{Policy Gradient}
Consider the problem of finding a good policy over a parameterized class, $\{\pi_{\btheta}: \btheta \in \Theta\}$. Each policy $\pi_{\btheta}: \Scal \to \Delta(\Acal)$ is stochastic and we assume that $\pi_\btheta(a|s)$ is differentiable w.r.t.~$\btheta$. Policy gradient algorithms \citep{williams1992simple} perform (stochastic) gradient descent on the objective $J(\pi_\btheta)$, and the following expression is an unbiased gradient based on a single trajectory \citep{sutton2000policy}:  
\begin{align}\label{eq:pg}
\sum_{t=0}^T \left(\gbt \log\pi_{\btheta}(a_t|s_t) \sum_{t'=t}^T \gamma^{t'} r_{t'}\right). 
\end{align}
Note that although most PG results are derived for the infinite-horizon discounted case, they can be immediately applied to our setup, since our formulation in Section~\ref{sec:mdp} can be turned into an infinite-horizon discounted MDP by treating $s_{T+1}$ as an absorbing state. 

\subsection{Further Notations} \label{sec:notation}
Since we always consider the estimators based on a single on-policy trajectory, all expectations $\EE[\cdot]$ are w.r.t.~that on-policy distribution induced by $\pi$ (for OPE) or $\pi_\btheta$ (for PG). Following the notations in \citet{jiang2016doubly}, we use $\EE_t[\cdot]$ as a shorthand for the conditional expectation $\EE[\cdot | s_0, a_0, \ldots, s_{t-1}, a_{t-1}]$, and similarly $\VV_t[\cdot]$ and $\cov_t[\cdot]$ for the conditional (co)variance. We will often see the usage $\EE_t[\cdot | s_t]$, which simply means $\EE[\cdot | s_0, a_0, \ldots, s_{t-1}, a_{t-1}, s_t]$.

\paragraph{Omitted function arguments} Since all value-functions of the form $V^\pi$ (or $Q^\pi$) are always applied on $s_t$ (or $s_t, a_t$) in the trajectory, we will sometimes omit such arguments and use  $V^\pi_t$ as a shorthand for $V^\pi(s_t)$ (and $Q^\pi_t$ for $Q^\pi(s_t, a_t)$). Similarly, we write $\pi_t$ as a shorthand for $\pi(a_t|s_t)$, and $\pi_\btheta^t$ as a shorthand for $\pi_\btheta(a_t|s_t)$.

\section{Warm-up: Deriving PG from IS} \label{sec:warmup}
In this section we show how the most common forms of PG can be derived from the corresponding IS estimators.
Although these results will be later subsumed by our main theorem in Section~\ref{sec:dr-pg}, it is still instructive to derive the connection between IS and PG from the simpler cases.

\paragraph{Vanilla PG}
\begin{restatable}{proposition}{StepISOPEPG}\label{prop:stepIS_ope2pg}
The standard PG (Eq.\eqref{eq:pg}) can be derived from taking finite difference over step-wise IS (Eq.\eqref{eq:step_is}).
\end{restatable}

\begin{proof}
Denote $\bm{e_i}$ as the i-th standard basis vectors in $\btheta=(\theta_1,\theta_2,...,\theta_d)\in\mathbb{R}^d$ space, and denote $\epsilon_i$ as a small scalar for $i=1,2,...,d$. Then, we apply step-wise IS on the policy $\pi' := \pai{i}$ for arbitrary $i=1,2,...,d$:
{\allowdisplaybreaks \begin{align*}
\jpi{\turb{i}} =&\sum_{t=0}^T \rho_{[0:t]} \gamma^t r_t = \sum_{t=0}^T\gamma^t r_t \prod_{t'=0}^t  \frac{\ptai{t'}{i}}{\ptb{t'}} \\
=&\sum_{t=0}^T\gamma^t r_t \prod_{t'=0}^t(1+\frac{\ptai{t'}{i}-\ptb{t'}}{\ptb{t'}}) \\
=&\sum_{t=0}^T\gamma^t r_t (1+\sum_{t'=0}^t\frac{\dpi{t'}{i}}{\ptb{t'}}) +o(\epsilon_i)\\
=&\sum_{t=0}^T\gamma^t r_t+\sum_{t=0}^T\frac{\dpi{t}{i}}{\ptb{t}}\sum_{t'=t}^T\gamma^{t'}r_{t'}+o(\epsilon_i).
\end{align*}}
where $\dpi{t}{i}$ is a shorthand of $\langle \gbt \pi_{\btheta}^{t}, \dti{i}\bm{e}_i \rangle$. Then,
\begin{align*}
\pt{\jpi{\btheta}}{i}=&\tpzi{i}\frac{\jpi{\turb{i}}-\jpi{\btheta}}{\dti{i}}\\
=&\tpzi{i}\sum_{t=0}^T\frac{\dpi{t}{i}/\ptb{t}}{\dti{i}} \sum_{t'=t}^T \gamma^{t'} r_{t'}\\
=&\sum_{t=0}^T\pt{\log\ptb{t}}{i} \sum_{t'=t}^T \gamma^{t'} r_{t'}.
\end{align*}
As a result, the estimator derived from Eq.(\ref{eq:step_is}) should be:
\begin{align*} &\Big(\pt{\jpi{\btheta}}{1},\pt{\jpi{\btheta}}{2},...,\pt{\jpi{\btheta}}{d}\Big)\trans\\
=&\sum_{t=0}^T \left(\gbt \log\ptb{t} \sum_{t'=t}^T \gamma^{t'} r_{t'}\right). \tag*{\qedhere} 
\end{align*}
\end{proof}

\paragraph{PG with a Baseline} Using a state baseline is a simple and popular form of variance reduction for PG. Below we show that there exists an unbiased OPE estimator (Eq.\eqref{pgwithbaseline:ope}) that yields such a PG estimator via the procedure in Eq.\eqref{eq:grad}.
\begin{restatable}{proposition}{BaselinePGOPE}\label{prop:baseline_pg2ope}
PG with a baseline \citep{greensmith2004variance}
\begin{align} \label{pgwithbaseline:grad}
\sum_{t=0}^T \left(\gbt \log \ptb{t} \Big(\sum_{t'=t}^{T}\gamma^{t'} r_{t'}-\gamma^t b(s_t)\Big)\right)
\end{align}
can be derived by taking finite difference over the following OPE estimator, 
\begin{align} \label{pgwithbaseline:ope}
\jpi{\pi'} = \sum_{t=0}^T \gamma^{t}\rho_{[0:t]}\Big( r_{t}-b(s_{t})+ \gamma  b(s_{t+1})\Big) + b(s_0).
\end{align}
(Recall that  $b(s_{T+1}) = 0$.) Furthermore, Eq.\eqref{pgwithbaseline:ope} is an unbiased estimator of $J(\pi')$.
\end{restatable}
We defer the proof to Appendix~\ref{appendix:table_missing_proof}, which also includes the connections between a few other pairs of OPE and PG estimators presented in Table~\ref{tab:big}.

\paragraph{Remark} Whenever we encounter unfamiliar OPE or PG estimators in the derivation, we always verify their unbiasedness from scratch. (We omit such verification for those well-known estimators.) However, a PG estimator that is derived from a known and unbiased OPE estimator should be automatically unbiased, thanks to the linearity (and hence exchangeability) of differentiation and expectation; see further details in Appendix~\ref{app:exchange}.

\section{A General Form of PG Derived From DR} \label{sec:dr-pg}
In the previous section we have derived some popular forms of PG from their IS counterparts. However, as \citet{cheng2019trajectory} noticed, the variance reduction in popular PG algorithms are relatively na\"ive. From our perspective, this is evidenced by the fact that the IS counterparts of these popular PG estimators---which often uses a ``baseline'' that carries the semantics of state-value functions---are na\"ive OPE estimators and do not fully exploit variance reduction opportunities in IS. 

In this section, 
we derive a very general form of PG from the unbiased estimator in the IS family that arguably performs the maximal amount of variance reduction, known as the doubly robust estimator \citep{jiang2016doubly, thomas2016data}, which requires an approximate Q-value function of $\pi_\btheta$, $\tilde{Q}^{\pi_\btheta}$. We show that a special case of the resulting PG estimator is exactly equivalent to that derived by \citet{cheng2019trajectory} recently from a control variate perspective. This special case treats $\tilde{Q}^{\pi_\btheta}$ as not varying with $\btheta$, whereas our more general estimator can further leverage the gradient information $\nabla \tilde{Q}^{\pi_\btheta}$ to reduce even more variance. Furthermore, the two popular forms of PG examined in Section~\ref{sec:warmup} are also subsumed as the special cases of our estimator. 

\subsection{Derivation of DR-PG}
\begin{restatable}{theorem}{DROPEPG}\label{thm:dr_ope2pg}
Let $\tv^{\pi_\btheta}_t(s):= \tq^{\pi_\btheta}_t(s, \pi_\btheta(s)) = \EE_{a \sim \pi_\btheta(s)}[\tq^{\pi_\btheta}_t(s, a)] $. The following estimator is an unbiased policy gradient that can be derived by taking finite difference over the doubly robust estimator for OPE:
\begin{align}\label{drpg:maintext:funcQ}
&\sum_{t=0}^T\Big\{\gbt\log\ptb{t}\Big[\sum_{t_1=t}^T\gamma^{t_1}r_{t_1}+\sum_{t_2=t+1}^T\gamma^{t_2}\Big(\tv^\pb_{t_2}-\tq^\pb_{t_2}\Big)\Big]\nonumber\\
&+\gamma^t\Big(\gbt\tv^\pb_t-\gbt\tq^\pb_t-\tq^\pb_t\gbt\log\ptb{t}\Big)\Big\}.
\end{align}
\end{restatable}

\begin{remark}
Since $\gbt\tv^\pb_t(s) = \gbt~  \EE_{a \sim \pi_\btheta(s)}[\tq^{\pb}_t(s, a)]$, the dependencies on $\pb$ in the subscript and the superscript will both contribute to the gradient calculation.\footnote{In contrast, the $\gbt\tv^\pb_t$ term in the estimator of  \citet{cheng2019trajectory} only differentiate w.r.t.~the subscript ($\EE_{a \sim \pi_\btheta(s)}$), as they treat $\tq^\pb_t$ as not varying with $\pb$.} 
\end{remark}

\begin{remark}
While $\tq^{\pb}$ and $\gbt \tq^{\pb}$ look related in notation, they have independent degrees of freedoms and can be estimated using separate procedures; see Appendix~\ref{app:1st_order} for further details.
\end{remark}

\begin{proof}
	We first show how Eq.(\ref{drpg:maintext:funcQ}) can be derived from DR. We start from the recursive form of DR \citep{jiang2016doubly}:
	\begin{align}
	\hd^{\pi'}_t = \tv_{t}^{\pi'} +\frac{\pi'_{t}}{\pi_{t}}\Big(r_t + \gamma \hd^{\pi'}_{t+1}-\tq_t^{\pi'}\Big).
	\end{align}
	where $\tq$ and $\tv$ are the approximate value functions, and $\tv^\pi_t=\sum_a\pi(a|s_t)\tq_t^\pi(s_t,a)$. Note that $\hd^{\pi'}_0$ is equivalent to the expanded form given in Eq.\eqref{eq:dr} \citep{thomas2016data}. Denote $e_i$ as the i-th standard basis vectors in $\btheta\in\mathbb{R}^d$ space, and denote $\epsilon_i$ as a scalar. Let $\dpi{t}{i}$ be the shorthand of $\langle \gbt \pi_{\btheta}^{t}, \dti{i}\bm{e}_i \rangle$. Then, we apply DR on the policy $\pi' := \pai{i}$ for $i=1,2,...,d$:
	
	\begin{align*}
	&\hd^\pai{i}_t - \hd^\pb_t \\
	=& \tv_{t}^\pai{i} - \tv_{t}^\pb+ \frac{\dpi{t}{i}}{\ptb{t}}\Big(r_t + \gamma \hd^\pb_{t+1}-\tq_t^\pb\Big)\\
	&+ \Big(\gamma \hd^\pai{i}_{t+1}-\gamma \hd^\pb_{t+1}-\tq_t^\pai{i}+\tq_t^\pb\Big)+o(\epsilon_i).
	\end{align*}
	Therefore,
	\begin{align*}
	\pt{\hd_t^\pb}{i}=&\tpzi{i}\frac{\hd_t^{\pai{i}}-\hd_t^\pb}{\dti{i}}\\
	=&\pt{\tv_t^\pb}{i}+ \pt{\log\ptb{t}}{i}(r_t+\gamma \hd_{t+1}^\pb-\tq_t^\pb)\\
	&+\gamma \pt{\hd_{t+1}^\pb}{i}-\pt{\tq_t^\pb}{i}.
	\end{align*}
	As a result,
	\begin{align*}\label{dr:gradient}
	\gbt\hd_t^\pb =& \Big(\pt{\hd_t^\pb}{1}, \pt{\hd_t^\pb}{2},..., \pt{\hd_t^\pb}{d}\Big)\trans\\
	=&\gbt\tv_{t}^\pb + \gbt\log\ptb{t}\Big(r_t + \gamma \hd^\pb_{t+1}-\tq_t^\pb\Big)\\
	&+\gamma \gbt \hd^\pb_{t+1} - \gbt\tq_t^\pb.  \numberthis
	\end{align*}
	We can continue to expand (\ref{dr:gradient}) and finally get the following estimator:
	\begin{align*}
	&\sum_{t=0}^T\Big\{\gbt\log\ptb{t}\Big[\sum_{t_1=t}^T\gamma^{t_1}r_{t_1}+\sum_{t_2=t+1}^T\gamma^{t_2}\Big(\tv^\pb_{t_2}-\tq^\pb_{t_2}\Big)\Big]\\
	&+\gamma^t\Big(\gbt\tv^\pb_t-\gbt\tq^\pb_t-\tq^\pb_t\gbt\log\ptb{t}\Big)\Big\}. 
	\end{align*}
	Next, we show that the estimator is unbiased.
	\begingroup
	\allowdisplaybreaks
	\begin{align*}
	&\me[\sum_{t=0}^T\Big\{\gbt\log\ptb{t}\Big[\sum_{t_1=t}^T\gamma^{t_1}r_{t_1}+\sum_{t_2=t+1}^T\gamma^{t_2}\Big(\tv^\pb_{t_2}-\tq^\pb_{t_2}\Big)\Big]\\
	&+\gamma^t\Big(\gbt\tv^\pb_t-\gbt\tq^\pb_t-\tq^\pb_t\gbt\log\ptb{t}\Big)\Big\}]\\
	=&\underbrace{\me\bigg[\sum_{t=0}^T\bigg(\gbt\log\ptb{t}\Big(\sum_{t_1=t}^T\gamma^{t_1}r_{t_1}\Big)\bigg)\bigg]}_{p_1}\\
	&+\me\bigg[\sum_{t=0}^T\bigg(\underbrace{\gbt\log\ptb{t}\Big[\sum_{t_2=t+1}^T\gamma^{t_2}\Big(\tv^\pb_{t_2}-\tq^\pb_{t_2}\Big)\Big]}_{p_2^t}\bigg)\bigg]\\
	&+\me\bigg[\sum_{t=0}^T\gamma^t\bigg(\underbrace{\gbt\tv^\pb_t-\frac{\gbt[\tq^\pb_t\ptb{t}]}{\ptb{t}}}_{p_3^t}\bigg)\bigg].
	\end{align*}
	\endgroup
	Since $p_1$ is the usual PG estimator, it suffices to show that $p_2^t$ and $p_3^t$ are equal to $0$ in expectation. For $p_2^t$,
	\begin{align*}
	\me_t[p_2^t]=&\me_t\Big[\gbt\log\ptb{t}\Big(\sum_{t_2=t+1}^T\gamma^{t_2}\me_{t_2}\Big[\tv^\pb_{t_2}-\tq^\pb_{t_2}\Big| s_{t_2}\Big]\Big)\Big]=0,
	\end{align*}
	where $\me_{t_2}\Big[\tv^\pb_{t_2}-\tq^\pb_{t_2}\Big| s_{t_2}\Big] = 0$ because PG is on-policy ($a_{t_2} \sim \pi_\btheta(s_{t_2})$). Similarly,
	\begin{align*}
	\me_t[p_3^t]=&\sum_a\pb(a|s_t)\gbt\tv^\pb_t-\sum_a \gbt[\tq^\pb_t\ptb{t}]\\
	=&\gbt\tv^\pb_t-\gbt[\sum_a\tq^\pb_t\ptb{t}]=0.\tag*{\qedhere}
	\end{align*}
\end{proof} 

It turns out that our estimator subsumes many previous ones as its special cases.  
\paragraph{Special case when $\tilde{Q}^{\pi_\btheta}$ is not a function of $\btheta$} 
When we treat $\tilde{Q}^{\pi_\btheta}$ not as a function of $\btheta$, i.e., $\nabla \tilde{Q}^{\pi_\btheta} \equiv 0$, $\nabla \tv^{\ptb{}}\equiv \sum_a \tq^{\ptb{}}\nabla \ptb{}$, the estimator becomes\footnote{Note that here $\gbt\tv^\pb_t$ is in general non-zero even when $\gbt\tq^\pb_t \equiv \mathbf{0}$, as $\gbt\tv^\pb_t$ additionally depends on $\btheta$ through the expectation over actions drawn from $\pi_{\btheta}$ when we convert $Q$-value to $V$-value.}
\begin{align}
\label{drpg:maintext:constantQ}
&\sum_{t=0}^T\Big\{\gbt\log\ptb{t}\Big[\sum_{t_1=t}^T\gamma^{t_1}r_{t_1}+\sum_{t_2=t+1}^T\gamma^{t_2}\Big(\tv^\pb_{t_2}-\tq^\pb_{t_2}\Big)\Big]\nonumber\\
&+\gamma^t\Big(\gbt\tv^\pb_t-\tq^\pb_t\gbt\log\ptb{t}\Big)\Big\},
\end{align}
which is exactly the same as the one given by \citet{cheng2019trajectory}. We will compare the variance of the two estimators below and discuss when our general form can reduce more  variance.  

\paragraph{Special case when $\tilde{Q}^{\pi_\btheta}(s, a)$ depends on neither $\btheta$ nor $a$} As a more restrictive special case, when $\tilde{Q}^{\pi_\btheta}$ is only a function of its state argument, we essentially recover the baseline method. This is obvious by comparing the correponding OPE estimator of PG with baseline to DR, and noticing that they are equivalent when we let $\tilde{Q}^{\pi_\btheta}(s,a) := b(s)$. 

\paragraph{Special case when $\tilde{Q}^{\pi_\btheta}\equiv 0$} As a further special case, when the approximate $Q$-value function is always $0$, we recover the standard PG estimator, which corresponds to step-wise IS. 

\subsection{Variance Analysis}
In this section, we analyze the variance of the DR-PG estimator given in Eq.(\ref{drpg:maintext:funcQ}). 

\begin{restatable}{theorem}{DRCovProofOne}\label{maintext:dr_cov_proof1}
The covariance matrix of the estimator Eq.(\ref{drpg:maintext:funcQ}) is
{\allowdisplaybreaks
\begin{align*}\label{drpg:maintext:variance_all}
&\me\bigg[\sum_{n=0}^T\gamma^{2n}\bigg(\mv_{n+1}[r_n]\Big(\sum_{t=0}^{n}\gbt\log\ptb{t}\Big)\Big(\sum_{t=0}^{n}\gbt\log\ptb{t}\Big)\trans\\
&\quad\quad\quad\quad+\cov_n\bigg[\gbt Q^\pb_{n}-\gbt\tq^\pb_{n}\\
&\quad\quad\quad\quad\quad\quad+\Big(\sum_{t=0}^{n}\gbt\log\ptb{t}\Big)\Big(Q^\pb_n-\tq^\pb_n\Big)\bigg|s_n\bigg]\\
&\quad\quad\quad\quad+\cov_n\bigg[\gbt V_n^\pb+\Big(\sum_{t=0}^{n-1}\gbt\log\ptb{t}\Big)V^\pb_n\bigg]\bigg)\bigg].\numberthis
\end{align*}}
where $\cov_n[\cdot]$ denotes the covariance matrix of a column vector (defined as $\cov_n[\vect v]:=\me_n[\vect v \vect v\trans]-\me_n[\vect v]\me_n[\vect v]\trans$), and we omit 0 in $\me_0$ and $\cov_0$.
\end{restatable}


We defer the proof to Appendix \ref{appendix:dr_cov_proof2}. Besides, since many common estimators are special cases of DR-PG, we can obtain their variances as direct corollaries of Theorem~\ref{thm:dr_ope2pg}.

\paragraph{Discussions} 
As we can see, the approximate value-function can help reduce the second term in Eq.\eqref{drpg:maintext:variance_all} when both $Q^\pb$ and $\gbt Q^\pb$ are well approximated by $\tq^\pb$ and $\gbt\tq^\pb$ respectively. Comparing with \citet{cheng2019trajectory}, which is our special case with $\gbt\tq^\pb \equiv \bm 0$,  we can see that as long as $\gbt \tq^\pb$ is a better approximation of $\gbt Q^{\pb}$ than $\bm 0$, the new estimator will generally have lower variance than the previous one. We note that such a situation is very common in variance reduction by control variates; we refer the readers to \citet{jiang2016doubly} for very similar discussions when they compare DR to step-wise IS. 

\subsection{Cramer-Rao Lower Bound for Policy Gradient}
We now state the Cramer-Rao lower bound for policy gradient, which is a lower bound for \emph{any} unbiased estimator for the PG problem. As we will see, the DR-PG estimator achieves the C-R bound of PG when the MDP has a tree structure and both $\tq^\pb$ and $\gbt \tq^\pb$ are accurate, a property inherited directly from the DR estimator in OPE \citep[Theorem 2]{jiang2016doubly}. As a special case, when we further assume that the environment is fully deterministic (but the policy can still be stochastic), \textbf{DR-PG is the only estimator that achieves $\mathbf{0}$ variance with accurate side information}, and other estimators have non-zero variance in general (see Table~\ref{tab:big}). 

\begin{theorem}[Informal]
For tree-structured MDPs (i.e., each state only appears at a unique time step and can be reached by a unique trajectory), the Cramer-Rao lower bound of PG is	
\begin{align*}
\me\Big[\sum_{t=0}^T\gamma^{2t}\bigg\{\mv_{t+1}[r_t]\Big[\Big(\sum_{t_1=0}^t\frac{\partial \log\ptb{t_1}}{\partial \theta_i}\Big)\Big]^2+ \\
\mv_t\Big[\Big(V_t^\pb\sum_{t_1=0}^{t-1}\frac{\partial \log\ptb{t_1}}{\partial \theta_i}+\frac{\partial V_t^\pb}{\partial \theta_i}\Big)\Big]\bigg\}\Big],
\end{align*}
which coincides with the variance of DR-PG when $\tq^\pb\equiv Q^\pb$ and $\gbt\tq^\pb \equiv \gbt Q^\pb$. 
\end{theorem}

Please refer to Appendix~\ref{app:CR} for the formal definitions and theorem statements, where we prove the more general results for DAG MDPs (Theorem~\ref{thm:dag}) and induce the lower bound for tree-MDPs as a direct corollary (Remark~\ref{rem:achieve}). 

\subsection{Practical Considerations} \label{sec:practical}
It is worth pointing out that the new estimator requires more information than $\tq^\pb$: it also requires $\gbt\tq^\pb$, which is sometimes not available, e.g., when $\tq^\pb$ is obtained by applying a model-free algorithm on a separate dataset. However, when an approximate dynamics model of the MDP is available (as considered by \citet{jiang2016doubly, cheng2019trajectory}), both $\tq^\pb$ and $\gbt\tq^\pb$ can be computed by running simulations in the approximate model \emph{for each data point}, where the former can be estimated by Monte-Carlo and the latter can be estimated by the PG estimators. Since we need to draw multiple trajectories starting from each $(s_t, a_t)$ in the dataset, the approach will be computationally intensive and not suitable for situations where the original problem is also a simulation. The new estimator is most likely useful when the bottleneck is the sample efficiency in the real environment and computation in the approximate model is relatively cheap.

Despite the computational intensity, in the next section we provide proof-of-concept experimental results showing the variance reduction benefits of the new estimator compared to prior baselines. 

\section{Experiments}
In this section we empirically validate the effectiveness of DR-PG. Most of our experiment settings follow exactly from \citet{cheng2019trajectory} (we reuse their code). 
\begin{figure*}[h!]
	\centering
	\includegraphics[scale=0.35]{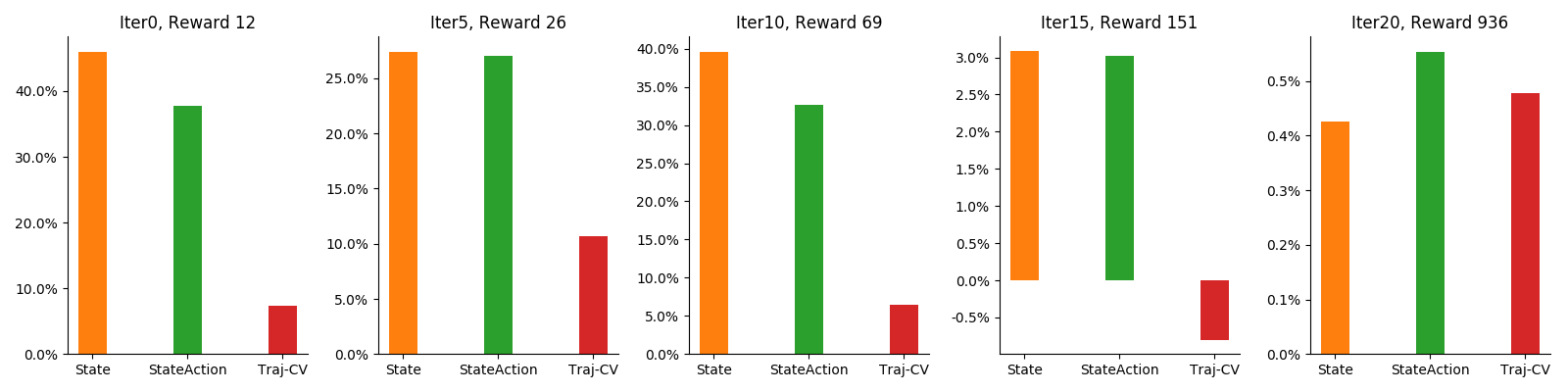}
	\caption{The variance reduction ratio of DR-PG comparing with (1) State-dependent baseline (orange), (2) State-action-dependent baseline (green), (3) Trajectory-wise control variate (red). We omit the results of Standard PG, because DR-PG enjoys a great variance reduction ratio more than 60\% in each iteration. Y-axes show the variance reduction ratio. In each sub-title, we indicate the iteration number and the evaluation results of the policy at that iteration.  }\label{fig:var}
\end{figure*}

\begin{figure}[h!]
	\centering
	\includegraphics[scale=0.25]{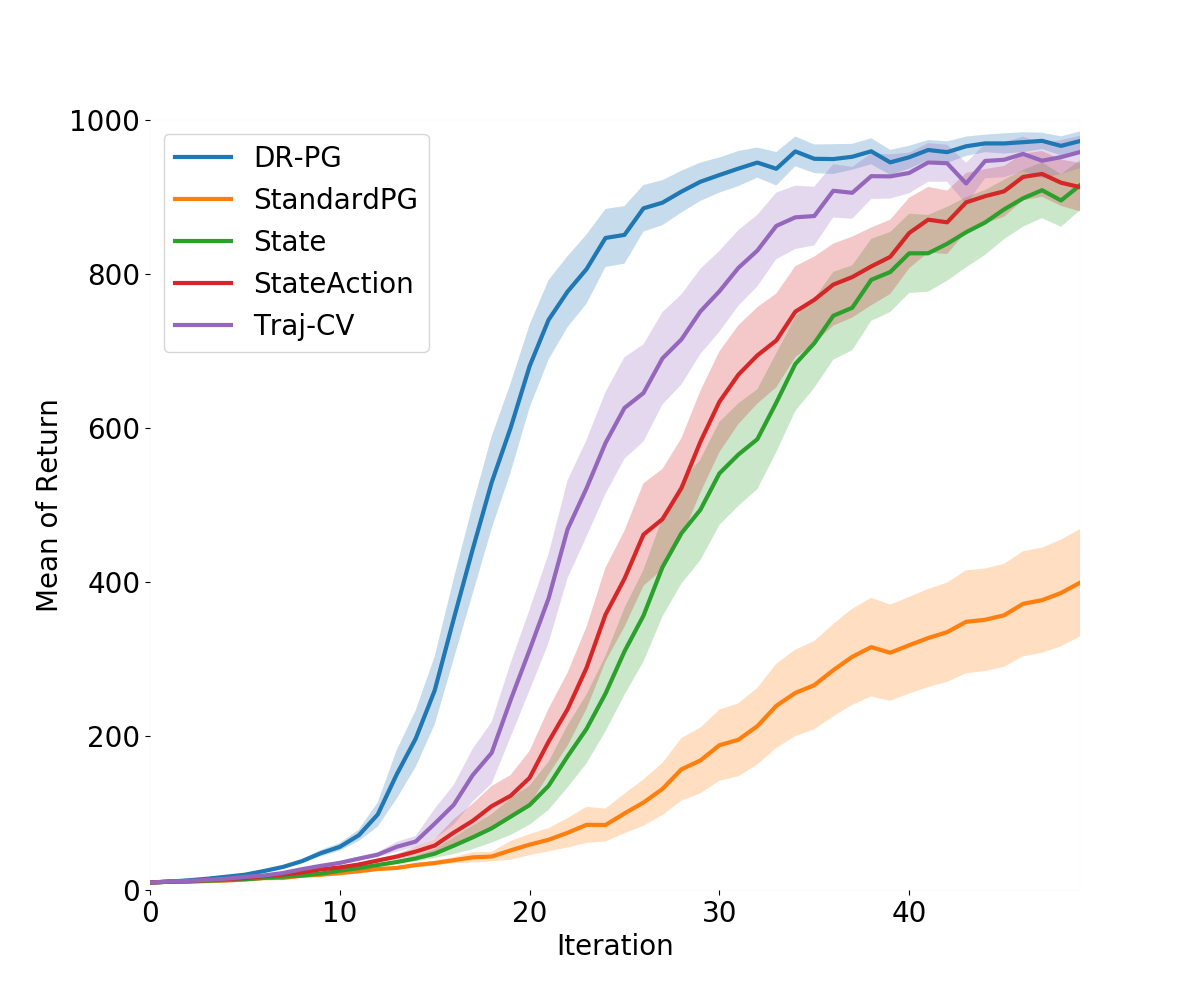}
	\caption{Comparison of different PG estimators in policy optimization. Y-axes show the mean of the expected returns of the policies over 150 trials learned by different PG methods. Error bars show double the standard errors, which correspond to 95\% confidence intervals. \label{fig:opt}}
\end{figure}

\subsection{Setup}
\paragraph{Compared Methods}
We empirically demonstrate the variance reduction effect and the optimization results of the new DR-PG estimator, and compare it to the following methods: (a) Standard PG, (b) Standard PG with state-dependent baseline,  (c) Standard PG with state-action-dependent baseline, (d) Standard PG with trajectory-wise baseline. 
For simplicity we will drop the prefix ``Standard PG'' when referring to the methods (b)--(d). See Appendix \ref{exp:true_estimators} for the detailed implementations of these methods.

\paragraph{Environments and Approximate Models}
We use the CartPole Environment in OpenAI Gym \citep{brockman2016openai} with
DART physics engine \citep{DBLP18}, and set the horizon length to be 1000. We follow \citet{cheng2019trajectory} for the choice of neural network architecture and training methods in building the policy $\pi$, value function estimator $\tilde{V}$, and dynamics model $\tilde{d}$. 


\paragraph{Implementation of the Estimators}
For state-baseline, we use $\tilde{V}$ as the baseline function. For the state-action-baseline, Traj-CV, and DR-PG, we additionally need  $\tq(s_t,a_t)$, which is computed by the combination of $\tilde{V}$ and $\tilde{d}$: $\tilde{Q}(s_t,a_t):=r(s_t,a_t)+\delta \tilde{V}(\tilde{d}(s_t, a_t))$, where $\delta$ is a hyperparameter that plays the role of discount factor introduced by \citet{cheng2019trajectory}.

For each state $s_t$ we use Monte Carlo (1000 samples) to compute the expectation (over $a_t$) of $Q(s_t,a_t)\nabla\log\pi(s_{t'},a_{t'})$, where $t'=t$ for state-action baseline and $t'=t,t+1,...,T$ for Traj-CV and our method. All these design choices are taken from \citet{cheng2019trajectory} as-is.

Our DR-PG requires additional estimation of $\nabla Q^{\ptb{}}$ and its expectation $\me_{\ptb{}}[\nabla Q^{\ptb{}}]$, and we obtain them by Monte Carlo. To estimate $\nabla Q^{\ptb{}}(s,a)$, we sample $n_q$ trajectories with $\pi$ and $\tilde{d}$, starting from $(s,a)$ with the maximum length no larger than $L$. Since we are solving another policy gradient problem now (one in the biased dynamics model), we choose state-baseline with $\tv$ as a computationally-cheap variance reduction method to speed up the computation (c.f.~Section~\ref{sec:practical}), and use a different discount factor $\gamma'$ (again for variance reduction \citep{baxter2001infinite, jiang2015dependence}). As for $\me_{\ptb{}}[\nabla Q^{\ptb{}}(s,a)]$, we first sample $n_v$ actions at state $s$, and then use the above procedure to estimate $\nabla Q^{\ptb{}}(s,a_i)$ for $i\in\{1,2,3,\ldots,n_v\}$. Finally, we compute the mean of them as the expectation. We choose $n_q=20,n_v=20,L=30$ and $\gamma'=0.9$ in the actual experiments. We observe that even with relatively small $n_q$, $n_v$, there is already significant variance reduction effect for our method. See Appendix~\ref{exp:true_estimators} for further details.

\subsection{Variance Reduction Comparison}
We first compare the variance reduction benefits of DR-PG against several baseline methods. The variance reduction ratio is defined as $\frac{\hat{\mv}_G-\hat{\mv}_{DR}}{\hat{\mv}_{G}}$, 
where $\hat{\mv{}}_G$ denotes the sum of the policy gradients estimator $G$'s variance over all 194 parameters (i.e., the trace of the covariance matrix), and $G$ can be Standard PG, state-dependent baseline, state-action-dependent baseline, or trajectory-wise control variate.

The results are shown in Figure \ref{fig:var}. To display the variance reduction results in different training stages, we first train a randomly initialized agent with DR-PG for multiple iterations and stop the training when the policy is near optimal (nearly 20 iterations in total in this case). Every 5 iteration, we save the policy $\pi_i$, as well as the value function estimator $\tilde{V}_i$ and dynamic estimator $\tilde{d}_i$. 

For each ($\pi_i$, $\tilde{V}_i$, $\tilde{d}_i$) tuple, we first use 10000 sampled trajectories to build state-dependent baseline estimators and compute the mean of as the (estimated) groundtruth. (We use state-dependent baseline because this method suffers less variance than standard PG and is computationally cheap compared to other control variate methods). Next, we sample another 500 trajectories and calculate the mean squared error w.r.t.~the approximate true gradient mentioned above for each gradient estimator, which gives an estimation of the estimators' variance (since all estimators considered here are unbiased). 

As we can see in Figure 1, DR-PG has better variance reduction effect than the other methods. At the initial training stage (Iterations 0-10), DR-PG can be uniformly better than the others, and the variance reduction ratio is quite large. 
When the policy is close to optimal (Iterations 15-20), the complexity of the true value function will increase, as the trajectories will last longer and the true values of different states become much more distinct than before. As a result, it is more difficult for the value function estimator to make accurate predictions, hence the estimation of $\tq^{\pi}$ and $\nabla\tq^{\pi}$ are less accurate and the variance reduction ratio decreases. However, our method still has advantage over the others in these cases. Moreover, we will show in the next section that such decrease in the final training stage will not stop DR-PG from achieving a near-optimal policy with less amount of data from real environment, i.e., attaining better sample efficiency.

\subsection{Policy Optimization}

In this experiment, we directly compare the policy optimization performance of different PG methods, i.e., they generate different sequences of policies now (as opposed to computing the gradient for the same policies as in the previous experiment). In each iteration, we record the mean of the accumulated rewards over 5 trajectories sampled from true environment, and use them to represent the performance of the policy. We repeat multiple trials of the entire experiment with different random seeds and plot the mean expected return in Figure \ref{fig:opt}. 
As can be clearly seen, DR-PG achieves a near-optimal performance with a significantly less amount of data drawn from the real environment compared to the baseline algorithms. 

\subsection{Computational Cost} 
To understand the computational overhead of our method due to having to compute $\tilde{Q}$ and $\nabla\tilde{Q}$, we compare the computational cost of applying the PG estimators to train the policy till near-optimal, when the policy value averaged over 5 trajectories exceeds 900 for the first time (the optimal return is around 1000). The total CPU/GPU usage is reported in Figure \ref{fig:usage}, where we omit the standard PG because it costs much more  than the others due to extended number of training iterations. As we can see, DR-PG  requires a reasonable amount of additional computational resources compared to the other estimators.
\begin{figure}[h!]
	\centering
	\includegraphics[scale=0.5]{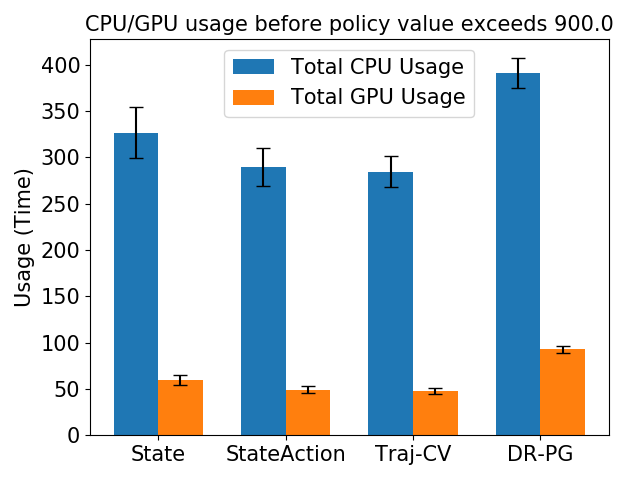}
	\caption{Comparison of computational costs. Y-axes show the total CPU/GPU usage. Error bars show double the standard errors. \label{fig:usage}}
\end{figure}

\section{Conclusion}
This paper investigates a direct connection between variance reduction techniques for on-policy policy gradient and for off-policy evaluation with importance sampling. From the DR estimator for OPE, we derive a very general form of PG that subsumes many previous estimators as special cases, and achieve more variance reduction in the ideal situation with accurate side information.

\section*{Acknowledgement}
The research project is motivated by a question asked by Lihong Li in 2015.  The authors gratefully thank Ching-An Cheng and his coauthors for the code of the trajectory-wise control variates method and the valuable comments. 

\newpage

\bibliography{reference}
\bibliographystyle{icml2020}

\onecolumn
\appendix

\clearpage

\section{Results in Table \ref{tab:big} Not Presented in the Main Text}\label{appendix:table_missing_proof}

\BaselinePGOPE*

\begin{proof}
Notice that
\begin{align*}
\gamma^t b(s_t) =& \gamma^t b(s_t) + \sum_{t'=t+1}^{T+1} \gamma^{t'}\Big(b(s_{t'})-b(s_{t'})\Big)=\sum_{t'=t}^{T}\Big(\gamma^{t'}b(s_{t'})- \gamma^{t'+1} b(s_{t'+1}))\Big).\numberthis\label{pgwithbaseline:sumLemma}
\end{align*}
where we use the fact that $b(s_{T+1})=0$. Use (\ref{pgwithbaseline:sumLemma}) to replace $\gamma^t b(s_t)$ in (\ref{pgwithbaseline:grad}) and we obtain that
\begin{align*}
(\ref{pgwithbaseline:grad}) =& \sum_{t=0}^T \left(\gbt \log \ptb{t} \Big(\sum_{t'=t}^{T}\gamma^{t'} r_{t'}-\gamma^t b(s_t)\Big)\right)\\
=&\sum_{t=0}^T \left(\gbt \log \ptb{t} \Big(\sum_{t'=t}^{T}\gamma^{t'} r_{t'}-\sum_{t'=t}^{T}\Big(\gamma^{t'}b(s_{t'})- \gamma^{t'+1}b(s_{t'+1})\Big)\Big)\right)\\
=&\sum_{t=0}^T \left(\gbt \log \ptb{t} \Big(\sum_{t'=t}^{T}\gamma^{t'}\Big( r_{t'}-b(s_{t'})+ \gamma b(s_{t'+1})\Big)\Big)\right)\\
=&\sum_{t=0}^T \left( \gamma^{t}\Big( r_{t}-b(s_{t})+ \gamma b(s_{t+1})\Big)\Big(\gbt\log\prod_{t'=0}^{t} \ptb{t'}\Big) \right).\numberthis\label{pgwithbaseline:derive}
\end{align*}

Now we show that (\ref{pgwithbaseline:derive}) can be generated by (\ref{pgwithbaseline:ope}), we similarly add a disturbance 
\begin{align*}
\jpi{\turb{i}}=&\sum_{t=0}^T\gamma^t \prod_{t'=0}^t\frac{\ptai{t'}{i}}{\ptb{t'}} \Big( r_{t}-b(s_{t})+ \gamma b(s_{t+1})\Big)+b(s_0)\\
=&\sum_{t=0}^T\gamma^t \Big( r_{t}-b(s_{t})+ \gamma b(s_{t+1})\Big) (1+\sum_{t'=0}^t\frac{\dpi{t'}{i}}{\ptb{t'}})+b(s_0) +o(\epsilon_i).
\end{align*}
where $\dpi{t}{i}$ is a shorthand of $\langle \gbt \pi_{\btheta}^{t}, \dti{i}\bm{e}_i \rangle$.
Then we calculate the partial derivative for each $i=1,2,...,d$,
\begin{align*}
\pt{\jpi{\btheta}}{i}=&\tpzi{i}\frac{\jpi{\turb{i}}-\jpi{\btheta}}{\dti{i}}\\
=&\sum_{t=0}^T\gamma^t \Big( r_{t}-b(s_{t})+ \gamma b(s_{t+1})\Big) \sum_{t'=0}^t\pt{\log\ptb{t}}{i}.
\end{align*}
As a result, the policy gradient estimator derived from (\ref{pgwithbaseline:ope}) should be
$$
\Big(\pt{\jpi{\btheta}}{1},\pt{\jpi{\btheta}}{2},...,\pt{\jpi{\btheta}}{d}\Big)\trans=\sum_{t=0}^T \left( \gamma^{t}\Big( r_{t}-b(s_{t})+ \gamma b(s_{t+1})\Big)\Big(\gbt\log\prod_{t'=0}^{t} \ptb{t'}\Big) \right).
$$
Finally we prove that (\ref{pgwithbaseline:ope}) is an unbiased OPE estimator. 
\begin{align*}
&\me\Big[\sum_{t=0}^T \gamma^{t}\rho_{[0:t]}\Big( r_{t}-b(s_{t})+ \gamma b(s_{t+1})\Big)+b(s_0)\Big]\\
=&\me\Big[\sum_{t=0}^T\gamma^{t}\rho_{[0:t]}r_t\Big]+\me\Big[\sum_{t=0}^T\rho_{[0:t-1]}\Big(b(s_t)-\rho_t b(s_t)\Big)\Big]+\me\Big[\gamma^{T+1}b(s_{T+1})\Big]\\
=&V^{\pi'}_0+\sum_{t=0}^T\me\Big[\rho_{[0:t-1]}b(s_t)\me_t[1-\rho_t|s_t]\Big] + 0 = V^{\pi'}_0. \tag*{\qedhere} 
\end{align*}
\end{proof}


\begin{proposition}\label{prop:traj_ope2pg}
REINFORCE \cite{williams1992simple}
\begin{equation}
\sum_{t=0}^T \left(\gbt \log\ptb{t} \sum_{t'=0}^T \gamma^{t'} r_{t'}\right)
\end{equation}
can be derived by taking finite difference over trajectory-wise IS
\begin{equation}\label{appendix:Traj_IS_PG}
 \rho_{[0:T]} \sum_{t=0}^T  \gamma^t r_t
\end{equation}
where $\rho_{[0:T]}$ is the cumulative product.
\end{proposition}

\begin{proof}
Denote $e_i$ as the i-th standard basis vectors in $\btheta\in\mathbb{R}^d$ space, and denote $\epsilon_i$ as a scalar. Besides, use $\dpi{t}{i}$ as a shorthand of $\langle \gbt \pi_{\btheta}^{t}, \dti{i}\bm{e}_i \rangle$. Then, we apply trajectory-wise IS on the policy $\pi' := \pai{i}$ for arbitrary $i=1,2,...,d$:
\begin{align*}
\jpi{\turb{i}} =&\rho_{[0:T]} \sum_{t=0}^T  \gamma^t r_t = \sum_{t=0}^T\gamma^t r_t \prod_{t'=0}^T\frac{\ptai{t'}{i}}{\ptb{t'}}\\
=&\sum_{t=0}^T\gamma^tr_t \prod_{t'=0}^T(1+\frac{\ptai{t'}{i}-\ptb{t'}}{\ptb{t'}}) \\
=&\sum_{t=0}^T\gamma^tr_t + \sum_{t=0}^T\frac{\dpi{t}{i}}{\ptb{t}}\sum_{t'=0}^T\gamma^{t'}r_{t'}+o(\epsilon_i).
\end{align*}
Then,
\begin{align*}
\pt{\jpi{\btheta}}{i}=&\tpzi{i}\frac{\jpi{\turb{i}}-\jpi{\btheta}}{\dti{i}}\\
=&\tpzi{i}\sum_{t=0}^T\frac{\dpi{t}{i}/\ptb{t}}{\dti{i}} \sum_{t'=t}^T \gamma^{t'} r_{t'}\\
=&\sum_{t=0}^T\pt{\log\ptb{t}}{i} \sum_{t'=0}^T \gamma^{t'} r_{t'}
\end{align*}
Therefore, the estimator derived from (\ref{appendix:Traj_IS_PG}) should be 
\begin{align*}
&\Big(\pt{\jpi{\btheta}}{1},\pt{\jpi{\btheta}}{2},...,\pt{\jpi{\btheta}}{d}\Big)\trans\\
=&\sum_{t=0}^T\left(\Big(\pt{\log\ptb{t}}{1},\pt{\log\ptb{t}}{2},...,\pt{\log\ptb{t}}{d}  \Big)\trans \sum_{t'=0}^T \gamma^{t'} r_{t'} \right)\\
=&\sum_{t=0}^T \left(\gbt \log\ptb{t} \sum_{t'=0}^T \gamma^{t'} r_{t'}\right) \numberthis\label{perstepIS:fullyexpand} 
\end{align*}
\end{proof}

\begin{proposition}\label{prop:ac_pg2ope}
Recall the policy gradient estimator in Actor-Critic Algorithm
\begin{equation}\label{ac:pgform} 
\sum_{t=0}^T\gamma^t\gbt\log\pb(a_t|s_t) f^w(s_t,a_t)
\end{equation}
where $f^w(\cdot, \cdot)$ is the critic function parameterized by $w$, and we assume $f^w(s_{T+1}, a_{T+1})=0$ for any $(s_{T+1}, a_{T+1})\in \cS\times\cA$. Then Eq.(\ref{ac:pgform}) can be derived by taking finite difference over the following OPE estimator:
\begin{equation}
\jpi{\pi'}=\sum_{t=0}^T\gamma^t\rho_{[0:t]}\Big(f^w(s_{t}, a_{t})-\gamma f^w(s_{t+1}, a_{t+1})\Big). \label{ac:estimator}
\end{equation}
\end{proposition}
\begin{proof}
Given (\ref{ac:pgform}), we can rewrite $\gamma^t f^w(s_t, a_t)$ as:
\begin{align*}
&\gamma^t f^w(s_t, a_t) \\
= & \gamma^t f^w(s_t, a_t) + \sum_{t'=t}^{T} \gamma^{t'+1} \Big( f^w(s_{t'+1}, a_{t'+1})-  f^w(s_{t'+1}, a_{t'+1})\Big)\\
=&\sum_{t'=t}^{T} \Big(\gamma^{t'} f^w(s_{t'}, a_{t'}) - \gamma^{t'+1}  f^w(s_{t'+1}, a_{t'+1})\Big) + \gamma^{T+1}f^w(s_{T+1}, a_{T+1})\\
=&\sum_{t'=t}^{T}  \Big(\gamma^{t'}f^w(s_{t'}, a_{t'}) -  \gamma^{t'+1}f^w(s_{t'+1}, a_{t'+1})\Big).
\end{align*}
The last equation is because the fact that $f^w(s_{T+1},a_{T+1})=0$ for any $(s_{T+1},a_{T+1})\in \cS\times\cA$.
Then, we can rewrite (\ref{ac:pgform}) to
\begin{align*} &\sum_{t=0}^T\gamma^t\gbt\log\pb(a_t|s_t) f^w(s_t,a_t)\\
=&\sum_{t=0}^T\gbt\log\pb(a_t|s_t)\sum_{t'=t}^{T} \Big(\gamma^{t'} f^w(s_{t'}, a_{t'}) -  \gamma^{t'+1} f^w(s_{t'+1}, a_{t'+1})\Big)\\
=&\sum_{t=0}^T\gamma^t \Big(\sum_{t'=0}^{t}\gbt\log\pb(a_{t'}|s_{t'})\Big)\Big( f^w(s_{t}, a_{t})-\gamma f^w(s_{t+1}, a_{t+1})\Big).\numberthis\label{ac:grad}\\
\end{align*}
To see how (\ref{ac:grad}) can be generated by (\ref{ac:estimator}), we similarly add a small disturbance on $\btheta$ along the direction of $\bm{e}_i$:
\begin{align*}
\jpi{\turb{i}}=&\sum_{t=0}^T\gamma^t\prod_{t'=0}^t \frac{\ptai{t'}{i}}{\ptb{t'}}\Big(f^w(s_{t}, a_{t})-\gamma f^w(s_{t+1}, a_{t+1})\Big)\\
=&\sum_{t=0}^T\gamma^t(1+\sum_{t'=0}^t\frac{\dpi{t'}{i}}{\ptb{t'}})\Big(f^w(s_{t}, a_{t})-\gamma f^w(s_{t+1}, a_{t+1})\Big)+o(\epsilon_i).
\end{align*}
Then follow the same steps as Proposition \ref{prop:traj_ope2pg} and finally obtain (\ref{ac:grad}).
\end{proof}

\section{On the Unbiasedness of PG Derived from Unbiased OPE Estimators} \label{app:exchange}
Here we show that an PG estimator derived by differentiating a knowingly unbiased OPE estimator is always unbiased (despite that we verified the unbiasedness of DR-PG in Theorem~\ref{thm:dr_ope2pg} independently). 

\begin{proposition}
Suppose $\hv^{\pi}$ is an unbiased OPE estimator for any target policy $\pi$. The PG estimator obtained by  $\nabla_{\btheta'} \hv^{\pi_{\btheta'}}|_{\btheta'=\btheta}$  is also unbiased.
\end{proposition}
\begin{proof} Let $p(\tau|\pi_{\btheta})$ denote the probability of seeing trajectory $\tau$ when taking actions according to policy $\pi_{\btheta}$. 
	\begin{align*}
	&\EE_{\tau\sim p(\tau|\pi_\btheta)}[\nabla_{\btheta'} \hv^{\pi_{\btheta'}}|_{\btheta'=\btheta}]
	=\sum_\tau p(\tau|\pi_\btheta) (\nabla_{\btheta'} \hv^{\pi_{\btheta'}})|_{\btheta'=\btheta}\\
	=&(\nabla_{\btheta'} \sum_\tau p(\tau|\pi_\btheta)  \hv^{\pi_{\btheta'}})|_{\btheta'=\btheta}
	=(\nabla_{\btheta'}\EE_{\tau\sim p(\tau|\pi_\btheta)}[\hv^{\pi_{\btheta'}}])|_{\btheta'=\btheta}\\
	=&\nabla_{\btheta'} V_0^{\pi_{\btheta'}}|_{\btheta'=\btheta}=\gbt J(\pi_{\btheta}). \tag*{($\hv$ is unbiased) \qedhere}
	\end{align*}
\end{proof}

\section{Clarification on the Relationship between $\tq^\pb$ and $\gbt \tq^\pb$} \label{app:1st_order}

To explain why $\tq^\pb$ and $\gbt \tq^\pb$ do not have to be related in any ways, we temporarily deviate from the notations in the main text and adopt more complete albeit lengthy notations for clarification purposes. For now we let $\btheta'$ be the free variable that represents the policy parameters, and $\btheta$ be the constant vector that represents the \emph{current} policy we are computing gradient for. The goal of PG is then to estimate
$
\nabla_{\btheta'} J(\pi_{\btheta'}) |_{\btheta' = \btheta}.
$ 
In the main text we omit ``$|_{\btheta' = \btheta}$'' and overload $\btheta$ with $\btheta'$, which is standard, but here distinguishing between them helps explanation. Now consider the $\tq^\pb$ and $\gbt\tq^\pb$ terms in Eq.\eqref{drpg:maintext:funcQ}; they are actually
$
\tq^{\pi_{\btheta'}} |_{\btheta' = \btheta}
$ 
and
$
\nabla_{\btheta'}\tq^{\pi_{\btheta'}} |_{\btheta' = \btheta}
$ 
respectively. It should be clear then, that these two terms are the 0-th order and the 1-st order information of $\tq^{\pi_{\btheta'}}$, respectively,  \emph{at the current policy parameters $\btheta$}, so they have independent degrees of freedom. In practice, we often only specify the 0-th and the 1-st order information without ever fully specifying how $\tq^{\pi_{\btheta'}}$ changes with $\btheta'$, and the object $\tq^{\pi_{\btheta'}}$ is introduced purely for mathematical convenience in our derivation.

\section{Proof of Theorem \ref{maintext:dr_cov_proof1}}\label{appendix:dr_cov_proof2}
Below we first provide a relatively concise proof of this theorem; an alternative proof based on recursion and induction is given later.
\DRCovProofOne*

\begin{proof}
We first rewrite DR-PG estimator (Eq.(\ref{drpg:maintext:funcQ})) into an equivalent form:
\begin{align*}
\sum_{t=0}^T\Big\{\gbt\log\ptb{t}\Big[\sum_{t'=t}^T\gamma^{t'}\Big(r_{t'}+\tv^\pb_{t'}-\tq^\pb_{t'}\Big)\Big]+\gamma^t\Big(\gbt\tv^\pb_t-\tv_t^\pb\gbt\log\ptb{t}-\gbt\tq^\pb_t\Big)\Big\}.\numberthis\label{drpg:eq_form}
\end{align*}
Define sequence $X(n)$ as
\begin{align*}
X(n)=\sum_{t=n}^T\Big\{\gbt\log\ptb{t}\Big[\sum_{t'=t}^T\gamma^{t'}\Big(r_{t'}+\tv^\pb_{t'}-\tq^\pb_{t'}\Big)\Big]+\gamma^t\Big(\gbt\tv^\pb_t-\tv_t^\pb\gbt\log\ptb{t}-\gbt\tq^\pb_t\Big)\Big\}.
\end{align*}
According to the definition of $X(n)$, we observe that $X(n)$ is the DR-PG estimator in step $n$, obtained by dropping the first $n$ component of the summation in (\ref{drpg:eq_form}). Then we have
\begin{align}
\me_{n}[X(n)]=&\gamma^n\gbt Q^\pb_{n-1}\label{drpg:maintext:expxn}\\
\me_n[X(n)|s_n]=&\gamma^n\gbt V^\pb_n \label{drpg:maintext:expxn_sn}
\end{align}
Define sequence $Y(n_1,n_2)$ as 
\begin{align*}
Y(n_1,n_2)=
\sum_{t=0}^{n_1}\Big\{\gbt\log\ptb{t}\Big[\sum_{t'=n_2}^T\gamma^{t'}\Big(r_{t'}+\tv^\pb_{t'}-\tq^\pb_{t'}\Big)\Big]\Big\}+\sum_{t=n_2}^{n_1}\gamma^{t}\Big(\gbt\tv^\pb_{t}-\tv_{t}^\pb\gbt\log\ptb{{t}}-\gbt\tq^\pb_{t}\Big).
\end{align*}
We take the convention that $\sum_{t=n_2}^{n_1}(\cdot) = 0$ when $n_1 < n_2$. According to the definition of $Y$,
\begin{align*}
\me_{n+1}[Y(n,n)]=&\me_{n+1}\bigg[\sum_{t=0}^{n}\Big\{\gbt\log\ptb{t}\Big[\sum_{t'=n}^T\gamma^{t'}\Big(r_{t'}+\tv^\pb_{t'}-\tq^\pb_{t'}\Big)\Big]\Big\}+\gamma^{n}\Big(\gbt\tv^\pb_{n}-\tv_{n}^\pb\gbt\log\ptb{{n}}-\gbt\tq^\pb_{n}\Big)\bigg]\\
=&\gamma^n\me_{n+1}\bigg[\Big(\sum_{t=0}^{n}\gbt\log\ptb{t}\Big)\Big(\sum_{t'=n}^T\gamma^{t'-n}r_{t'}-\tq_n^\pb\Big)+\tv^\pb_{n}\Big(\sum_{t=0}^{n-1}\gbt\log\ptb{t}\Big)+\gbt\tv^\pb_{n}-\gbt\tq^\pb_{n}\bigg]\\
&+\gamma^n\me_{n+1}\bigg[\Big(\sum_{t=0}^{n}\gbt\log\ptb{t}\Big)\Big(\sum_{t'=n+1}^T\gamma^{t'-n}\Big[\tv^\pb_{t'}-\tq^\pb_{t'}\Big]\Big)\bigg]\\
=&\gamma^n\me_{n+1}\bigg[\Big(\sum_{t=0}^{n}\gbt\log\ptb{t}\Big)\Big(Q^\pb_n-\tq^\pb_n\Big)+\tv^\pb_{n}\Big(\sum_{t=0}^{n-1}\gbt\log\ptb{t}\Big)+\gbt\tv^\pb_{n}-\gbt\tq^\pb_{n}\bigg]. 
\end{align*}
Moreover,
\begin{align*}
\me_{n}[Y(n-1,n)|s_n]=&\me_{n}\bigg[\Big(\sum_{t=0}^{n-1}\gbt\log\ptb{t}\Big)\Big(\sum_{t'=n}^T\gamma^{t'}r_{t'}\Big)\bigg|s_n\bigg]+\me_{n}\bigg[\Big(\sum_{t=0}^{n-1}\gbt\log\ptb{t}\Big)\Big(\sum_{t'=n}^T\gamma^{t'}(\tv^\pb_{t'}-\tq^\pb_{t'})\Big)\bigg|s_n\bigg] \\
=&\gamma^n\Big(\sum_{t=0}^{n-1}\gbt\log\ptb{t}\Big)V^\pb_n. 
\end{align*}
Define sequence $Z(n)$ as
\begin{align*}
Z(n)=X(n)+Y(n-1,n).
\end{align*}
According to this definition, it is easy to verify that
$$
Z(n)=X(n+1)+Y(n,n).
$$
Besides, $Z(0)$ is exactly the DR-PG estimator at step 0 and $Z(T+1)=0$.
Next, we consider the variance of $Z(n)$ given $s_0,a_0,...s_{n-1},a_{n-1}$. Use the law of the total variance (since the expectation of each component is independent), we have that
\begin{align*}
\cov_n[Z(n)]=\me_n[\cov_{n+1}[Z(n)]]+\me_n[\cov_n(\me_{n+1}[Z(n)]|s_n)]+\cov_n(\me_n[Z(n)|s_n]).
\end{align*}
We take a look at these three terms one by one. 
\paragraph{The First Term}
\begin{align*}
\me_n[\cov_{n+1}[Z(n)]] =& \me_n[\cov_{n+1}[Z(n+1)+Y(n,n)-Y(n,n+1)]]\\
=&\me_n\Big[\cov_{n+1}\Big[Z(n+1)+\sum_{t=0}^{n}\Big\{\gbt\log\ptb{t}\Big[\gamma^{n}\Big(r_{n}+\tv^\pb_{n}-\tq^\pb_{n}\Big)\Big]\Big\}\\
&+\gamma^{n}\Big(\gbt\tv^\pb_{n}-\tv_{n}^\pb\gbt\log\ptb{{n}}-\gbt\tq^\pb_{n}\Big)\Big]\Big]\\
=&\me_n[\cov_{n+1}[Z(n+1)]]+\gamma^{2n}\me_n[\cov_{n+1}[r_n\sum_{t=0}^{n}\gbt\log\ptb{t}]]\\
=&\me_n[\cov_{n+1}[Z(n+1)]]+\gamma^{2n}\me_n\Big[\mv_{n+1}[r_n]\Big(\sum_{t=0}^{n}\gbt\log\ptb{t}\Big)\Big(\sum_{t=0}^{n}\gbt\log\ptb{t}\Big)\trans\Big].
\end{align*}
In the last but two equation, we dropped the terms that are deterministic conditioned on $s_0,a_0,...,s_n,a_n$, and used the fact that the randomness of reward is independent of the randomness in the transition. 
\paragraph{The Second Term}
\begin{align*}
\me_n[\cov_n(\me_{n+1}[Z(n)]|s_n)]=&\me_n[\cov_n(\me_{n+1}[X(n+1)+Y(n,n)]|s_n)]\\
=&\gamma^{2n}\me_n\Big[\cov_n\Big[\gbt Q^\pb_{n}-\gbt\tq^\pb_{n}+\Big(\sum_{t=0}^{n}\gbt\log\ptb{t}\Big)\Big(Q^\pb_n-\tq^\pb_n\Big)\\
&+\gbt\tv^\pb_{n}+\tv^\pb_{n}\Big(\sum_{t=0}^{n-1}\gbt\log\ptb{t}\Big)\Big|s_n\Big]\Big]\\
=&\gamma^{2n}\me_n\Big[\cov_n\Big[\gbt Q^\pb_{n}-\gbt\tq^\pb_{n}+\Big(\sum_{t=0}^{n}\gbt\log\ptb{t}\Big)\Big(Q^\pb_n-\tq^\pb_n\Big)\Big|s_n\Big]\Big].
\end{align*}
Similarly, in the last step we dropped the terms that are deterministic conditioned on $s_0,a_0,...,s_{n-1},a_{n-1}$.
\paragraph{The Third Term}

\begin{align*}
\cov_n(\me_n[Z(n)|s_n])=&\cov_n(\me_n[X(n)+Y(n-1,n)|s_n])\\
=&\gamma^{2n}\cov_n[\gbt V_n^\pb+\Big(\sum_{t=0}^{n-1}\gbt\log\ptb{t}\Big)V^\pb_n].
\end{align*}
After combining all the expressions, we have
\begin{align*}
\cov_n[Z(n)]=&\me_n[\cov_{n+1}[Z(n+1)]]+\gamma^{2n}\me_n\Big[\mv_{n+1}[r_n]\Big(\sum_{t=0}^{n}\gbt\log\ptb{t}\Big)\Big(\sum_{t=0}^{n}\gbt\log\ptb{t}\Big)\trans\Big]\\
&+\gamma^{2n}\me_n\Big[\cov_n\Big[\gbt Q^\pb_{n}-\gbt\tq^\pb_{n}+\Big(\sum_{t=0}^{n}\gbt\log\ptb{t}\Big)\Big(Q^\pb_n-\tq^\pb_n\Big)\Big|s_n\Big]\Big] \\ 
&+\gamma^{2n}\cov_n\Big[\gbt V_n^\pb+\Big(\sum_{t=0}^{n-1}\gbt\log\ptb{t}\Big)V^\pb_n\Big].
\end{align*}
The theorem follows by expanding this equation for $Z(0)$, which is the DR-PG estimator.
\end{proof}

\paragraph{Alternative Proof} Below we show an alternative proof to Theorem \ref{maintext:dr_cov_proof1} via recursion and induction.

\begin{lemma}[A Simple Recursion]\label{lem:dr_one_step_cov}
Recall the per-step version of the gradient
\begin{align*}
\gbt\hd^\pb_t =\gbt\tv_{t}^\pb + \frac{\gbt\Big(\ptb{t}[r_t + \gamma \hd^\pb_{t+1}-\tq_t^\pb]\Big)}{\ptb{t}}.
\end{align*}
Then we have
\begin{align*}
\cov_t[\gbt\hd^\pb_t]=&\gamma^2\me_t\Big[\cov_{t+1}\Big[\frac{\gbt\Big(\ptb{t}\hd^\pb_{t+1}\Big)}{\ptb{t}}\Big]\Big]\\
&+\me_t\Big[\cov_{t+1}\Big[r_t\gbt\log\ptb{t}\Big]\Big]+\cov_t[\gbt V_{t}^\pb]\\
&+\me_t\Big[\cov_t\Big[(Q_t^\pb-\tq_t^\pb)\gbt\log\ptb{t}+\gbt Q_t^\pb-\gbt\tq_t^\pb\Big|s_t\Big]\Big].
\end{align*}
\end{lemma}

\begin{proof}
For simplicity, we will use $\trp{\bm v}$ to denote the outer product ${\bm v\bm v\trans}$, where ${\bm v}$ is the column vector. Then we have:

\begingroup
\allowdisplaybreaks
\begin{align*}
&\cov_t[\gbt \hd_t^\pb]=\me_t [\trp{\gbt \hd_t^\pb}]-\trp{\me_t[\gbt \hd_t^\pb]}\\
=&\me_t\Big[\trp{\underbrace{\frac{\gbt\Big(\ptb{t}\Big[r_t + \gamma \hd^\pb_{t+1}-Q_t^\pb\Big]\Big)}{\ptb{t}}}_{p1}+\underbrace{\gbt \tv_{t}^\pb +\frac{\gbt\Big(\ptb{t}\Big[Q_t^\pb-\tq_t^\pb\Big]\Big)}{\ptb{t}}}_{p2}}\Big]-\trp{\me_t[\gbt V_t^\pb]}\numberthis \label{dr:split1}\\
=&\me_t\Big[\trp{\underbrace{\frac{\gbt\Big(\ptb{t}\Big[r_t + \gamma \hd^\pb_{t+1}-Q_t^\pb\Big]\Big)}{\ptb{t}}}_{p1}}\Big]+\me_t\Big[\trp{\underbrace{\gbt \tv_{t}^\pb +\frac{\gbt\Big(\ptb{t}\Big[Q_t^\pb-\tq_t^\pb\Big]\Big)}{\ptb{t}}}_{p2}}\Big]-\trp{\me_t[\gbt V_t^\pb]}\numberthis\label{dr:split1result}\\
=&\gamma^2\me_t\Big[\cov_{t+1}\Big[\frac{\gbt\Big(\ptb{t}\hd^\pb_{t+1}\Big)}{\ptb{t}}\Big]\Big]+\me_t\Big[\cov_{t+1}[r_t\gbt\log\ptb{t}]\Big]-\trp{\me_t[\gbt V_t^\pb]}\\
&+\me_t\Big[\trp{
\gbt V_{t}^\pb + \underbrace{\gbt \tv_{t}^\pb-
\gbt V_{t}^\pb +\frac{\gbt\Big(\ptb{t}\Big[Q_t^\pb-\tq_t^\pb\Big]\Big)}{\ptb{t}}}_{p3}}\Big]\numberthis\label{dr:var1}\\
=&\gamma^2\me_t\Big[\cov_{t+1}[\frac{\gbt\Big(\ptb{t}\hd^\pb_{t+1}\Big)}{\ptb{t}}]\Big]+\me_t\Big[\cov_{t+1}[r_t\gbt\log\ptb{t}]\Big]+\me_t[\trp{
\gbt V_{t}^\pb}]-\trp{\me_t[\gbt V_t^\pb]}\\
&+\me_t\Big[\trp{\underbrace{\gbt \tv_{t}^\pb-
\gbt V_{t}^\pb +\frac{\gbt\Big(\ptb{t}\Big[Q_t^\pb-\tq_t^\pb\Big]\Big)}{\ptb{t}}}_{p3}}\Big]\numberthis\label{dr:split2}\\
=&\gamma^2\me_t\Big[\cov_{t+1}\Big[\frac{\gbt\Big(\ptb{t}\hd^\pb_{t+1}\Big)}{\ptb{t}}\Big]\Big]+\me_t[\cov_{t+1}[r_t\gbt\log\ptb{t}]]+\cov_t[\gbt V_{t}^\pb]+\underbrace{\me_t\Big[\cov_t\Big[\frac{\gbt[\ptb{t}(Q_t^\pb-\tq_t^\pb)]}{\ptb{t}}\Big|s_t\Big]\Big]}_{p4}.\numberthis\label{dr:recursive:final}\\
\end{align*}
\endgroup
From (\ref{dr:split1}) to (\ref{dr:split1result}), we use the fact
\begin{align*}
&\me_{t}\Big[\frac{\gbt\Big(\ptb{t}[r_t+\gamma\hd_{t+1}^\pb-Q_t^\pb]\Big)}{\ptb{t}}\Big|s_t,a_t\Big]\\
=&\me_{t}\Big[\Big(r_t+\gamma\hd_{t+1}^\pb-Q_t^\pb\Big)\frac{\gbt\ptb{t}}{\ptb{t}}\Big|s_t,a_t\Big]+\gamma\me_{t}[\Big(\gbt\hd_{t+1}^\pb-\gbt \me_{s_{t+1}}V_{t+1}^\pb\Big)\Big|s_t,a_t\Big]\\
=&\me_{t}\Big[\Big(r_t+\gamma V_{t+1}^\pb-Q_t^\pb\Big)\frac{\gbt\ptb{t}}{\ptb{t}}\Big|s_t,a_t\Big]+\gamma\me_{t}\Big[\Big(\gbt V_{t+1}^\pb-\gbt \me_{s_{t+1}}[V_{t+1}^\pb]\Big)\Big|s_t,a_t\Big]= \bm {0}.
\end{align*}
From (\ref{dr:split1result}) to (\ref{dr:var1}), we point out the first term is actually a variance term, and then split out the randomness of reward. From (\ref{dr:var1}) to (\ref{dr:split2}), we use the fact that $\me_t[p_3|s_t]=\bm 0$. As a result, $p_3$ can be considered as another covariance term and we use $p_4$ to represent it.
\end{proof}

\begin{lemma}[Generally Unbiased]\label{lem:generallyunbiased}
Given any $0\leq k\leq T$, for any $0\leq t \leq T-k$, we have
\begin{equation}
\me_{t+k}[\gbt\Big(\hd^\pb_{t+k}\prod_{t'=0}^{k-1}\ptb{t+t'}\Big)|s_{t+k}]=\gbt\Big(V^\pb_{t+k}\prod_{t'=0}^{k-1}\ptb{t+t'}\Big).\label{dr:prop1}
\end{equation}
\end{lemma}

\begin{proof}
\begin{align*}
&\me_{t+k}[\gbt\Big(\hd^\pb_{t+k}\prod_{t'=0}^{k-1}\ptb{t+t'}\Big)|s_{t+k}]\\
=&\me_{t+k}[\hd^\pb_{t+k}\gbt\prod_{t'=0}^{k-1}\ptb{t+t'}|s_{t+k}] + \me_{t+k}[\Big(\prod_{t'=0}^{k-1}\ptb{t+t'}\Big)\gbt\hd^\pb_{t+k}|s_{t+k}]\\
=&V^\pb_{t+k}\gbt\prod_{t'=0}^{k-1}\ptb{t+t'} + \Big(\prod_{t'=0}^{k-1}\ptb{t+t'}\Big)\gbt V^\pb_{t+k}\\
=&\gbt\Big(V^\pb_{t+k}\prod_{t'=0}^{k-1}\ptb{t+t'}\Big).
\end{align*}
where we use the fact that both $\hd^\pb_t$ and $\gbt\hd^\pb_t$ are unbiased for any $0\leq t\leq T$.
\end{proof}

\begin{lemma}[General Gradient Estimator]\label{lemma:generalEstimator}
Given any $0\leq k\leq T$, for any $0\leq t \leq T-k$, we have:
\begin{align*}
&\frac{\gbt\Big(\hd_{t+k}^\pb\prod_{t'=0}^{k-1}\ptb{t+t'}\Big)}{\prod_{t'=0}^{k-1}\ptb{t+t'}}\\
&\quad\quad\quad=\frac{\gbt\Big(\Big[r_{t+k}+\gamma\hd^\pb_{t+k+1}-\tq^\pb_{t+k}\Big]\prod_{t'=0}^{k}\ptb{t+t'}\Big)}{\prod_{t'=0}^{k}\ptb{t+t'}}+\frac{\gbt(\tv_{t+k}^\pb\prod_{t'=0}^{k-1}\ptb{t+t'})}{\prod_{t'=0}^{k-1}\ptb{t+t'}}.\numberthis\label{eq:generalEstimator}
\end{align*}
\end{lemma}
\begin{proof}
\begingroup
\allowdisplaybreaks
\begin{align*}
&\frac{\gbt\Big(\hd_{t+k}^\pb\prod_{t'=0}^{k-1}\ptb{t+t'}\Big)}{\prod_{t'=0}^{k-1}\ptb{t+t'}}\\
=&\gbt \hd^\pb_{t+k}+\hd_{t+k}^\pb\frac{\gbt\Big(\prod_{t'=0}^{k-1}\ptb{t+t'}\Big)}{\prod_{t'=0}^{k-1}\ptb{t+t'}}\\
=&\gbt\tv_{t+k}^\pb + [r_{t+k} + \gamma \hd^\pb_{t+k+1}-\tq_{t+k}^\pb]\frac{\gbt\Big(\ptb{t+k}\Big)}{\ptb{t+k}}+\gbt[r_{t+k} + \gamma \hd^\pb_{t+k+1}-\tq_{t+k}^\pb]\\
&+[\tv_{t+k}^\pb+r_{t+k} + \gamma \hd^\pb_{t+k+1}-\tq_{t+k}^\pb]\frac{\gbt\Big(\prod_{t'=0}^{k-1}\ptb{t+t'}\Big)}{\prod_{t'=0}^{k-1}\ptb{t+t'}}\\
=&[r_{t+k} + \gamma \hd^\pb_{t+k+1}-\tq_{t+k}^\pb]\Big(\sum_{t'=0}^{k-1}\gbt\log\ptb{t+t'}+\gbt\log\ptb{t+k}\Big)+\gbt[r_{t+k}+\gamma \hd^\pb_{t+k+1}-\tq_{t+k}^\pb]\\
&+\gbt\tv_{t+k}^\pb + \tv_{t+k}^\pb\frac{\gbt\Big(\prod_{t'=0}^{k-1}\ptb{t+t'}\Big)}{\prod_{t'=0}^{k-1}\ptb{t+t'}}\\
=&[r_{t+k} + \gamma \hd^\pb_{t+k+1}-\tq_{t+k}^\pb]\frac{\gbt\Big(\prod_{t'=0}^{k}\ptb{t+t'}\Big)}{\prod_{t'=0}^{k}\ptb{t+t'}}+\frac{\prod_{t'=0}^{k}\ptb{t+t'}}{\prod_{t'=0}^{k}\ptb{t+t'}}\gbt[r_{t+k} + \gamma \hd^\pb_{t+k+1}-\tq_{t+k}^\pb]\\
&+\frac{\prod_{t'=0}^{k-1}\ptb{t+t'}}{\prod_{t'=0}^{k-1}\ptb{t+t'}}\gbt\tv_{t+k}^\pb + \tv_{t+k}^\pb\frac{\gbt\Big(\prod_{t'=0}^{k-1}\ptb{t+t'}\Big)}{\prod_{t'=0}^{k-1}\ptb{t+t'}}\\
=&\frac{\gbt\Big(\Big[r_{t+k}+\gamma\hd^\pb_{t+k+1}-\tq^\pb_{t+k}\Big]\prod_{t'=0}^{k}\ptb{t+t'}\Big)}{\prod_{t'=0}^{k}\ptb{t+t'}}+\frac{\gbt(\tv_{t+k}^\pb\prod_{t'=0}^{k-1}\ptb{t+t'})}{\prod_{t'=0}^{k-1}\ptb{t+t'}}.
\end{align*}
\endgroup
\end{proof}

\begin{lemma}[The General Recursion]\label{proexpand}
Given any $0\leq k\leq T$, for any $0\leq t \leq T-k$, we have
\begin{align*}
&\cov_{t+k}\Big[\frac{\gbt\Big(\hd_{t+k}^\pb\prod_{t'=0}^{k-1}\ptb{t+t'}\Big)}{\prod_{t'=0}^{k-1}\ptb{t+t'}}\Big]\\
=&\gamma^2\me_{t+k}\Big[\cov_{t+k+1}\Big[\frac{\gbt\Big(\hd_{t+k+1}^\pb\prod_{t'=0}^{k}\ptb{t+t'}\Big)}{\prod_{t'=0}^{k}\ptb{t+t'}}\Big]\Big] + \me_{t+k}\Big[\cov_{t+k+1}\Big[\frac{r_{t+k}\gbt \prod_{t'=0}^{k}\ptb{t+t'}}{\prod_{t'=0}^{k}\ptb{t+t'}}\Big]\Big]\\
&+\me_{t+k}\Big[\cov_{t+k}\Big[\frac{\gbt\Big(\Big[Q^\pb_{t+k}-\tq^\pb_{t+k}\Big]\prod_{t'=0}^{k}\ptb{t+t'}\Big)}{\prod_{t'=0}^{k}\ptb{t+t'}}\Big|s_{t+k}\Big]\Big] +\cov_{t+k}\Big[\frac{\gbt\Big(V_{t+k}^\pb\prod_{t'=0}^{k-1}\ptb{t+t'}\Big)}{\prod_{t'=0}^{k-1}\ptb{t+t'}}\Big].\numberthis\label{dr:prop2}
\end{align*}
\end{lemma}

\begin{proof}
We use induction here. From Lemma \ref{lem:dr_one_step_cov}, we know (\ref{dr:prop2}) holds for any $t$ when $k=0$. Suppose we already know that when $k=k'-1$, it holds for any $0\leq t\leq T-k'+1$, next we will prove that when $k=k'$, it holds for $0\leq t\leq T-k'$. Since the limit of space, given a column vector ${\bm v}$, we use $\trp{{\bm v}}$ to denote ${\bm v}{\bm v}\trans$.
\begingroup
\allowdisplaybreaks
\begin{align*}
&\cov_{t+k'}\Big[\frac{\gbt\Big(\hd_{t+k'}^\pb\prod_{t'=0}^{k'-1}\ptb{t+t'}\Big)}{\prod_{t'=0}^{k'-1}\ptb{t+t'}}\Big]=\me_{t+k'}\Big[\trp{\frac{\gbt\Big(\hd_{t+k'}^\pb\prod_{t'=0}^{k'-1}\ptb{t+t'}\Big)}{\prod_{t'=0}^{k'-1}\ptb{t+t'}}}\Big]-\trp{\me_{t+k'}\Big[\frac{\gbt\hd_{t+k'}^\pb\prod_{t'=0}^{k'-1}\ptb{t+t'}}{\prod_{t'=0}^{k'-1}\ptb{t+t'}}\Big]}\\
=&\me_{t+k'}\Big[\trp{\frac{\gbt\Big(\Big[r_{t+k'}+\gamma\hd^\pb_{t+k'+1}-\tq^\pb_{t+k'}\Big]\prod_{t'=0}^{k'}\ptb{t+t'}\Big)}{\prod_{t'=0}^{k'}\ptb{t+t'}}+\frac{\gbt(\tv_{t+k'}^\pb\prod_{t'=0}^{k'-1}\ptb{t+t'})}{\prod_{t'=0}^{k'-1}\ptb{t+t'}}}\Big]\\
&-\trp{\me_{t+k'}\Big[\frac{\gbt\hd_{t+k'}^\pb\prod_{t'=0}^{k'-1}\ptb{t+t'}}{\prod_{t'=0}^{k'-1}\ptb{t+t'}}]}\numberthis\label{dr:expand_second_eq}\\
=&\me_{t+k'}[\trp{\underbrace{\frac{\gbt\Big(\Big[r_{t+k'}+\gamma\hd^\pb_{t+k'+1}-Q^\pb_{t+k'}\Big]\prod_{t'=0}^{k'}\ptb{t+t'}\Big)}{\prod_{t'=0}^{k'}\ptb{t+t'}}}_{p1}\\
&+\underbrace{\frac{\gbt(\tv_{t+k'}^\pb\prod_{t'=0}^{k'-1}\ptb{t+t'})}{\prod_{t'=0}^{k'-1}\ptb{t+t'}}+\frac{\gbt([Q_{t+k'}^\pb-\tq_{t+k'}^\pb]\prod_{t'=0}^{k'}\ptb{t+t'})}{\prod_{t'=0}^{k'}\ptb{t+t'}}}_{p2}}\Big]-\trp{\me_{t+k'}\Big[\frac{\gbt\hd_{t+k'}^\pb\prod_{t'=0}^{k'-1}\ptb{t+t'}}{\prod_{t'=0}^{k'-1}\ptb{t+t'}}\Big]}\numberthis\label{dr:expandp1p2}\\
=&\me_{t+k'}\Big[\trp{\underbrace{\frac{\gbt\Big(\Big[r_{t+k'}+\gamma\hd^\pb_{t+k'+1}-Q^\pb_{t+k'}\Big]\prod_{t'=0}^{k'}\ptb{t+t'}\Big)}{\prod_{t'=0}^{k'}\ptb{t+t'}}}_{p1}}\Big]\\
&+\me_{t+k'}\Big[\trp{\underbrace{\frac{\gbt(\tv_{t+k'}^\pb\prod_{t'=0}^{k'-1}\ptb{t+t'})}{\prod_{t'=0}^{k'-1}\ptb{t+t'}}+\frac{\gbt([Q_{t+k'}^\pb-\tq_{t+k'}^\pb]\prod_{t'=0}^{k'}\ptb{t+t'})}{\prod_{t'=0}^{k'}\ptb{t+t'}}}_{p2}}\Big]-\trp{\me_{t+k'}\Big[\frac{\gbt\hd_{t+k'}^\pb\prod_{t'=0}^{k'-1}\ptb{t+t'}}{\prod_{t'=0}^{k'-1}\ptb{t+t'}}\Big]}\numberthis\label{dr:expandp1p2div}\\
=&\gamma^2\me_{t+k'}\Big[\cov_{t+k'+1}\Big[\frac{\gbt\Big(\hd^\pb_{t+k'+1}\prod_{t'=0}^{k'}\ptb{t+t'}\Big)}{\prod_{t'=0}^{k'}\ptb{t+t'}}\Big]\Big] + \me_{t+k'}\Big[\cov_{t+k'+1}\Big[\frac{\gbt\Big(r_{t+k'}\prod_{t'=0}^{k'}\ptb{t+t'}\Big)}{\prod_{t'=0}^{k'}\ptb{t+t'}}\Big]\Big]\\
&+\me_{t+k'}\Big[\trp{\underbrace{\frac{\gbt([\tv_{t+k'}^\pb-V_{t+k'}^\pb]\prod_{t'=0}^{k'-1}\ptb{t+t'})}{\prod_{t'=0}^{k'-1}\ptb{t+t'}}+\frac{\gbt([Q_{t+k'}^\pb-\tq_{t+k'}^\pb]\prod_{t'=0}^{k'}\ptb{t+t'})}{\prod_{t'=0}^{k'}\ptb{t+t'}}}_{p3}+\frac{\gbt(V_{t+k'}^\pb\prod_{t'=0}^{k'-1}\ptb{t+t'})}{\prod_{t'=0}^{k'-1}\ptb{t+t'}}}\Big]\\
&-\trp{\me_{t+k'}\Big[\frac{\gbt\hd_{t+k'}^\pb\prod_{t'=0}^{k'-1}\ptb{t+t'}}{\prod_{t'=0}^{k'-1}\ptb{t+t'}}\Big]}\numberthis\label{dr:expandp3}\\
=&\gamma^2\me_{t+k'}\Big[\cov_{t+k'+1}\Big[\frac{\gbt\Big(\hd^\pb_{t+k'+1}\prod_{t'=0}^{k'}\ptb{t+t'}\Big)}{\prod_{t'=0}^{k'}\ptb{t+t'}}\Big]\Big] + \me_{t+k'}\Big[\cov_{t+k'+1}\Big[\frac{\gbt\Big(r_{t+k'}\prod_{t'=0}^{k'}\ptb{t+t'}\Big)}{\prod_{t'=0}^{k'}\ptb{t+t'}}\Big]\Big]\\
&+\me_{t+k'}\Big[\trp{\underbrace{\frac{\gbt([\tv_{t+k'}^\pb-V_{t+k'}^\pb]\prod_{t'=0}^{k'-1}\ptb{t+t'})}{\prod_{t'=0}^{k'-1}\ptb{t+t'}}+\frac{\gbt([Q_{t+k'}^\pb-\tq_{t+k'}^\pb]\prod_{t'=0}^{k'}\ptb{t+t'})}{\prod_{t'=0}^{k'}\ptb{t+t'}}}_{p3}}\Big]\\
&+\me_{t+k'}\Big[\trp{\frac{\gbt(V_{t+k'}^\pb\prod_{t'=0}^{k'-1}\ptb{t+t'})}{\prod_{t'=0}^{k'-1}\ptb{t+t'}}}]-\trp{\me_{t+k'}\Big[\frac{\gbt\hd_{t+k'}^\pb\prod_{t'=0}^{k'-1}\ptb{t+t'}}{\prod_{t'=0}^{k'-1}\ptb{t+t'}}\Big]}\numberthis\label{dr:expandp3div}\\
=&\gamma^2\me_{t+k'}\Big[\cov_{t+k'+1}\Big[\frac{\gbt\Big(\hd^\pb_{t+k'+1}\prod_{t'=0}^{k'}\ptb{t+t'}\Big)}{\prod_{t'=0}^{k'}\ptb{t+t'}}\Big]\Big] + \me_{t+k'}\Big[\cov_{t+k'+1}\Big[\frac{\gbt\Big(r_{t+k'}\prod_{t'=0}^{k'}\ptb{t+t'}\Big)}{\prod_{t'=0}^{k'}\ptb{t+t'}}\Big]\Big]\\
&+\me_{t+k'}\Big[\cov_{t+k'}\Big[\frac{\gbt([Q_{t+k'}^\pb-\tq_{t+k'}^\pb]\prod_{t'=0}^{k'}\ptb{t+t'})}{\prod_{t'=0}^{k'}\ptb{t+t'}}\Big|s_{t+k'}\Big]\Big]+\cov_{t+k'}[\frac{\gbt(V_{t+k'}^\pb\prod_{t'=0}^{k'-1}\ptb{t+t'})}{\prod_{t'=0}^{k'-1}\ptb{t+t'}}].\numberthis\label{dr:final}\\
\end{align*}
\endgroup
To obtain (\ref{dr:expand_second_eq}), we use \eqref{eq:generalEstimator} in Lemma \ref{lemma:generalEstimator}. From (\ref{dr:expandp1p2}) to (\ref{dr:expandp1p2div}), we use following resulting from Lemma \ref{lem:generallyunbiased}:
\begin{align*}
&\me_{t+k'}[p_1|s_{t+k'},a_{t+k'}]\\
=&\me_{t+k'+1}\Big[\frac{\gbt\Big(\Big[r_{t+k'}+\gamma\hd^\pb_{t+k'+1}-Q^\pb_{t+k'}\Big]\prod_{t'=0}^{k'}\ptb{t+t'}\Big)}{\prod_{t'=0}^{k'}\ptb{t+t'}}\Big]\\
=&\me_{t+k'+1}\Big[\frac{\gbt\Big(\Big[r_{t+k'}+\gamma V^\pb_{t+k'+1}-Q^\pb_{t+k'}\Big]\prod_{t'=0}^{k'}\ptb{t+t'}\Big)}{\prod_{t'=0}^{k'}\ptb{t+t'}}\Big]\\
=&\bm 0.\\
\end{align*}
and from (\ref{dr:expandp3}) to (\ref{dr:expandp3div}), we use the fact that
\begin{align*}
&\me_{t+k'}[p_3|s_{t+k'}]\\
=&\me_{t+k'}\Big[\frac{\gbt([\tv_{t+k'}^\pb-V_{t+k'}^\pb]\prod_{t'=0}^{k'-1}\ptb{t+t'})}{\prod_{t'=0}^{k'-1}\ptb{t+t'}}+\frac{\gbt([Q_{t+k'}^\pb-\tq_{t+k'}^\pb]\prod_{t'=0}^{k'}\ptb{t+t'})}{\prod_{t'=0}^{k'}\ptb{t+t'}}\Big|s_{t+k'}\Big]\\
=&\sum_{a_{t+k'}}\ptb{t+k'}\frac{\gbt([\tv_{t+k'}^\pb-V_{t+k'}^\pb]\prod_{t'=0}^{k'-1}\ptb{t+t'})}{\prod_{t'=0}^{k'-1}\ptb{t+t'}}+\sum_{a_{t+k'}}\ptb{t+k'}\frac{\gbt([Q_{t+k'}^\pb-\tq_{t+k'}^\pb]\prod_{t'=0}^{k'}\ptb{t+t'})}{\prod_{t'=0}^{k'}\ptb{t+t'}}\\
=&\frac{\gbt([\tv_{t+k'}^\pb-V_{t+k'}^\pb]\prod_{t'=0}^{k'-1}\ptb{t+t'})}{\prod_{t'=0}^{k'-1}\ptb{t+t'}}+\frac{\gbt(\Big[\sum_{a_{t+k'}}\ptb{t+k'}(Q_{t+k'}^\pb-\tq_{t+k'}^\pb)\Big]\prod_{t'=0}^{k'-1}\ptb{t+t'})}{\prod_{t'=0}^{k'-1}\ptb{t+t'}}\\
=&\bm 0.
\end{align*}
\end{proof}

\begin{theorem}[Restatement of Theorem \ref{maintext:dr_cov_proof1}]The covariance matrix of the estimator Eq.(\ref{drpg:maintext:funcQ}) is
\begin{align*}
\cov[\hd_0^\pb]=&\me\Bigg[\sum_{t=0}^T\gamma^{2t}\Bigg(\mv_{t+1}[r_t]\Big(\sum_{t'=0}^{t}\gbt\log\ptb{t'}\Big)\Big(\sum_{t'=0}^{t}\gbt\log\ptb{t'}\Big)\trans\\
&\quad\quad\quad\quad+\cov_t\Bigg[\gbt Q^\pb_{t}-\gbt\tq^\pb_{t}+\Big(\sum_{t'=0}^{t}\gbt\log\ptb{t'}\Big)\Big(Q^\pb_t-\tq^\pb_t\Big)\Bigg|s_t\Bigg]\\
&\quad\quad\quad\quad+\cov_t\Bigg[\gbt V_t^\pb+\Big(\sum_{t'=0}^{t-1}\gbt\log\ptb{t'}\Big)V^\pb_t\Bigg]\Bigg)\Bigg].\numberthis\label{dr:fullyexpansion2}
\end{align*}
\end{theorem}

\begin{proof}
Set $t=0$ in (\ref{dr:prop2}) of Lemma \ref{proexpand}. Then for $k=0,1,2,...,T$, we have the following $T+1$ equations:
\begin{align*}
\cov_{k}[\frac{\gbt\Big(\hd_{k}^\pb\prod_{t'=0}^{k-1}\ptb{t'}\Big)}{\prod_{t'=0}^{k-1}\ptb{t'}}]=&\gamma^2\me_{k}\Big[\cov_{k+1}\Big[\frac{\gbt\Big(\hd_{k+1}^\pb\prod_{t'=0}^{k}\ptb{t'}\Big)}{\prod_{t'=0}^{k}\ptb{t'}}\Big]\Big] + \me_{k}[\cov_{k+1}\Big[\frac{r_{k}\gbt \prod_{t'=0}^{k}\ptb{t'}}{\prod_{t'=0}^{k}\ptb{t'}}\Big]\Big]\\
&+\me_{k}\Big[\cov_{k}\Big[\frac{\gbt\Big(\Big[Q^\pb_{k}-\tq^\pb_{k}\Big]\prod_{t'=0}^{k}\ptb{t'}\Big)}{\prod_{t'=0}^{k}\ptb{t'}}\Big|s_{k}\Big]\Big] +\cov_{k}\Big[\frac{\gbt\Big(V_{k}^\pb\prod_{t'=0}^{k-1}\ptb{t'}\Big)}{\prod_{t'=0}^{k-1}\ptb{t'}}\Big].
\end{align*}
Notice that
$$
\me_{T}\Big[\cov_{T+1}\Big[\frac{\gbt\Big(\hd_{T+1}^\pb\prod_{t'=0}^{T}\ptb{t'}\Big)}{\prod_{t'=0}^{T}\ptb{t'}}\Big]\Big]={\bm O}.
$$
Combine all the above $T+1$ equations can we obtain that
\begin{align*}
\cov[\gbt\hd^\pb_0] =& \me_t\Big[ \sum_{t=0}^T\gamma^{2t}\Bigg(\mv_{t+1}[r_t]\Big(\sum_{t'=0}^{t}\gbt\log\ptb{t'}\Big)\Big(\sum_{t'=0}^{t}\gbt\log\ptb{t'}\Big)\trans\\
&\quad\quad\quad\quad+ \cov_{t}\Big[\frac{\gbt\Big(\Big[Q^\pb_{t}-\tq_{t}^\pb\Big] \prod_{t'=0}^{t}\ptb{t'}\Big)}{\prod_{t'=0}^{t}\ptb{t'}}\Big|s_{t}\Big]\\
&\quad\quad\quad\quad+\cov_{t}\Big[\frac{\gbt\Big(V_{t}^\pb\prod_{t'=0}^{t-1}\ptb{t'}\Big)}{\prod_{t'=0}^{t-1}\ptb{t'}}\Big]\Bigg)\Big]. \numberthis\label{dr:fullyexpansion}
\end{align*}
After applying the product rule of derivative and using the sum of $\gbt\log(\cdot)$ to replace the gradient ratio, we can get (\ref{dr:fullyexpansion2}).
\end{proof}

\section{Cramer Rao Lower Bound} \label{app:CR}
\begin{definition}[Discrete DAG MDP]
An MDP is a discrete Directed Acyclic Graph (DAG) MDP if:
\begin{itemize}
    \item The state space and the action space are finite.
    \item For any $s\in S$, there exists a unique $t\in \mathbb{N}$ such that, $\max_{\pi:\cS\rightarrow\cA} P(s_t=s|\pi)>0$. In other words, a state only occurs at a particular time step.
\end{itemize}
\end{definition}

\begin{theorem} \label{thm:dag}
For discrete DAG MDPs and a policy parameterized by $\btheta\in\mathbb{R}^d$, the variance of any unbiased estimator w.r.t. the $i$-th component of the policy gradient vector is lower bounded by 
\begin{align*}
&\me\Big[\sum_{t=0}^T\gamma^{2t}\bigg\{\mv_{r_t|s_t,a_t}[r_t]\Big[\sum_{\tau_{[0:t]}}\mathbb{I}[(s_t,a_t)\in\tau_{[0:t]}]\frac{P_{\pb}(\tau_{[0:t]})}{P_M(s_t,a_t)}\Big(\sum_{t_1=0}^t\frac{\partial \log\ptb{t_1}}{\partial \theta_i}\Big)\Big]^2\\
&\quad\quad\quad\quad+\mv_{s_t|s_{t-1},a_{t-1}}\Big[\sum_{\tau_{[0:t-1]}}\mathbb{I}[(s_{t-1},a_{t-1})\in\tau_{[0:t-1]}]\frac{P_\pb(\tau_{[0:t-1]})}{P_M(s_{t-1},a_{t-1})}\Big(V_t^\pb\sum_{t_1=0}^{t-1}\frac{\partial \log\ptb{t_1}}{\partial \theta_i}+\frac{\partial V_t^\pb}{\partial \theta_i}\Big)\Big]\bigg\}\Big].\numberthis\label{CR:lowerbound}
\end{align*}
\end{theorem}

\begin{proof}
We parameterize the MDP by $\mu(s_0)$, $P(s_{t+1}|s_t,a_t)$ and $P(r_t|s_t,a_t)$, for $t=0,1,2,...,T$. In convenience, we consider $\mu(s_0)$ as $P(s_0|\emptyset)$, so all the parameters can be represented as $P(s_t|s_{t-1},a_{t-1})$, (for t = 0 there is a single $s_{-1}$ and $a_{-1}$). These parameters are constraint by
\begin{align*}
\sum_{s_{t+1}}P(s_{t+1}|s_t,a_t)=1,~~~~\sum_{r_t}P(r_t|s_t,a_t)=1.
\end{align*}
\begin{equation}\label{CR:F_eta}
\left[ \begin{array}{ccccc}
\underbrace{1...1}_{s_0} & ~ & ~ & ~ &\\
~& \underbrace{1...1}_{r_0} & ~ & &\\
~&~& \ddots & ~ &\\
~&~& ~ & \underbrace{1...1}_{s_T} &\\
&~&~& ~ & \underbrace{1...1}_{r_T}\\
\end{array} 
\right ]\eta=\left[\begin{array}{c}
     1 \\
     1 \\
     \vdots  \\
     1 \\
     1
\end{array}\right].
\end{equation}
where $\eta_{s_t,a_t,s_{t+1}}=P(s_{t+1}|s_t,a_t)$ and $\eta_{s_t,a_t,r_t}=P(r_{t+1}|s_t,a_t)$, and we denote the matrix on the left hand side as $F$.
As mentioned in \cite{moore2010theory}, the constrained Cramer-Rao Bound is:
\begin{equation}
KU(U\trans I U)^{-1}U\trans K\trans.
\end{equation}
where $K$ is the Jacobian of the quantity we want to estimate, and $I$ is the Fisher Information Matrix computed by
\begin{equation}
I=\me\Big[\left(\frac{\partial \log P_\pb(\tau)}{\partial \eta}\right)\left(\frac{\partial \log P_\pb(\tau)}{\partial \eta}\right)\trans\Big].
\end{equation}
where $\tau=(s_0,a_0,r_0,...,s_T,a_T,r_T)$ is a sample trajectory under policy $\pb$ and $P_\pb(\tau)$ is the probability to obtain such a sample. 
$$
P_\pb(\tau)=\mu(s_0)\pb(a_0|s_0)P(r_0|s_0,a_0)P(s_1|s_0,a_0)...\pb(a_T|s_T)P(r_T|s_T,a_T).
$$
Particularly, we use $\tau_{[0:t]}$ to denote $(s_0,a_0,r_0,...,s_t,a_t,r_t)$, whose probability is
$$
P_\pb(\tau_{[0:t]})=\mu(s_0)\pb(a_0|s_0)P(r_0|s_0,a_0)P(s_1|s_0,a_0)...P(s_t|s_{t-1},a_{t-1})\pb(a_t|s_t)P(r_t|s_t,a_t).
$$
\paragraph{Calculate FIM $I$}
Define $g(\tau)$ as an indicator vector, with $g(\tau)_{s_t,a_t,s_{t+1}}=1$ if $(s_t,a_t,s_{t+1})\in \tau$, and $g(\tau)_{s_t,a_t,r_t}=1$ if $(s_t,a_t,r_t)\in \tau$. Then, we have
\begin{equation}
\frac{\partial \log P_\pb(\tau)}{\partial \eta}=\eta^{\circ -1}\circ g(\tau).
\end{equation}
where $\circ$ denotes the element-wise power/multiplication. Then we can rewrite $I$ to be

\begin{equation}
I=\me\Big[[\frac{1}{\eta_i\eta_j}]_{ij}\circ (g(\tau)g(\tau)\trans)\Big]=[\frac{1}{\eta_i\eta_j}]_{ij}\circ\me\Big[ g(\tau)g(\tau)\trans\Big].
\end{equation}
where $[\frac{1}{\eta_i\eta_j}]_{ij}$ is a matrix expressed by its (i,j)-th element. 
\begin{table}[!h]
    \centering
    \begin{tabular}{l|c}
    \toprule[1.0pt]
    Indexing Tuple & Elements Value\\
    \toprule[1.0pt]
    Diagonal   & $\frac{P_M(s_t,a_t)}{P(s_{t+1}|s_t,a_t)}$ or $\frac{P_M(s_t,a_t)}{P(r_t|s_t,a_t)}$\\
    \hline
    Row$(s_t,a_t,s_{t+1})$\quad Column$(s_{t'},a_{t'},s_{t'+1})$  & \multirow{2}{*}{$P_M(s_{t'},a_{t'})P_M(s_t,a_t|s_{t'+1})$}\\
    \cline{1-1}
    Row$(s_t,a_t,r_t)$\quad\quad Column$(s_{t'},a_{t'},s_{t'+1})$ & \\
    \hline
    Row$(s_t,a_t,r_{t})$\quad\quad Column$(s_{t'},a_{t'},r_{t'})$ & \multirow{2}{*}{$P_M(s_{t'},a_{t'})P_M(s_t,a_t|s_{t'},a_{t'})$}\\
    \cline{1-1}
    Row$(s_{t},a_{t},s_{t+1})$\quad Column$(s_{t'},a_{t'},r_{t'})$ & \\
    \hline
    Row$(s_t,a_t,r_t)$\quad \quad Column$(s_{t},a_{t},s_{t+1})$ & $P_M(s_t,a_t)$\\
    \hline
    Others & 0\\
    \toprule[1.0pt]
    \end{tabular}
    \caption{The elements of matrix $I$, where we assume $t'<t$ w.l.o.g.}
    \label{tab:CR:elements_of_I}
\end{table}

\noindent In the following, we provide the method to calculate the elements of $I$, the results are shown in the Table \ref{tab:CR:elements_of_I}.
\begin{itemize}
    \item We first take a look at the diagonal of $I$. Since the diagonal of $\me[g(\tau)g(\tau)\trans]$ consists of the marginal distribution $P(s_t,a_t,s_{t+1})$ and $P(s_t,a_t,r_t)$, the diagonal of $I$ should be $\frac{P_M(s_t,a_t)}{P(s_{t+1}|s_t,a_t)}$ or $\frac{P_M(s_t,a_t)}{P(r_t|s_t,a_t)}$, where we use $P_M$ to denote the marginal distribution. 

    \item Next, we calculate the element of $I$ whose row and column index are $(s_t,a_t,s_{t+1})$ and $(s_{t'},a_{t'},s_{t'+1})$, respectively. Notice that those non-diagonal elements with $t=t'$ equals to 0, w.l.o.g., we only consider the case with $t'<t$, and the entry should be $\frac{P_M(s_t,a_t,s_{t+1},s_{t'},a_{t'},s_{t'+1})}{P(s_{t+1}|s_t,a_t)P(s_{t'+1}|s_{t'},a_{t'})}=P_M(s_{t'},a_{t'})P_M(s_t,a_t|s_{t'+1})$. In fact, for those whose row and column indexing tuples are respectively $(s_t,a_t,r_t)$, $(s_{t'},a_{t'},r_{t'})$, we have a similar discussion. The only difference is that, $s_t,a_t$ do not depend on $r_{t'}$. Therefore, for those $t'<t$, the corresponding entry of $I$ should be $P_M(s_{t'},a_{t'})P_M(s_t,a_t|s_{t'},a_{t'})$.

    \item Then, let's focus on those case when both row and column are indexed by tuples $(s_t,a_t,r_t)$ and $(s_{t'},a_{t'},s_{t'+1})$ separately, with $t'\leq t$. For $t=t'$, the entry should be $P_M(s_t,a_t)$. For $t'<t$, then entry should be $\frac{P_M(s_t,a_t,r_t,s_{t'},a_{t'},s_{t'+1})}{P(r_t|s_t,a_t)P(s_{t'+1}|s_{t'},a_{t'})}=P_M(s_{t'},a_{t'})P_M(s_t,a_t|s_{t'+1})$. 
    
    \item Finally, as for those elements indexed by  $(s_{t},a_{t},s_{t+1})$ and $(s_{t'},a_{t'},r_{t'})$, with $t'<t$, we have $\frac{P_M(s_t,a_t,s_{t+1},s_{t'},a_{t'},r_{t'})}{P(s_{t+1}|s_t,a_t)P(r_{t'}|s_{t'},a_{t'})}=P_M(s_{t'},a_{t'})P_M(s_{t},a_{t}|s_{t'},a_{t'})$.
\end{itemize}

\paragraph{Calculate $(U\trans IU)^{-1}$}
We use a similar strategy as \citep{jiang2016doubly} to diagonalize $I$, in order to avoid taking inverse of non-diagonal matrix. Notice that
\begin{align*}
U\trans I U =U\trans (F\trans X\trans +I+XF)U.
\end{align*}
where $X$ can be arbitrary, and $F$ is the matrix on the l.h.s of (\ref{CR:F_eta}). Denote $(F\trans X\trans +I+XF)$ as $D$, our goal is to find a $X$ to make $D$ a diagonal matrix, whose diagonal is the same as $I$'s. Notice that $F\trans X\trans$ and $XF$ are symmetry, so we can design $XF$ to eliminate all non-zero value in the upper triangle of $I$ and keep the other components unchanged. Then $F\trans X\trans$ will eliminate the lower triangle part and do not change the rest. Easy to verify that we can choose such a $X$:
\begin{table}[!h]
    \centering
    \begin{tabular}{l|c}
    \toprule[1.0pt]
    Indexing Tuple($t'\neq t$) & Elements Value\\
    \toprule[1.0pt]
    Row$(s_{t'},a_{t'},s_{{t'}+1})$\quad Column$(s_{t},a_{t},1)$  & \multirow{2}{*}{$-P_M(s_{t'},a_{t'})P_M(s_t,a_t|s_{t'+1})\mathbb{I}(t'<t)$}\\
    \cline{1-1}
    Row$(s_{t'},a_{t'},s_{t'+1})$\quad Column$(s_{t},a_{t},2)$ & \\
    \hline
    Row$(s_{t'},a_{t'},r_{t'})$\quad\quad Column$(s_{t},a_{t},1)$ & \multirow{2}{*}{$-P_M(s_{t'},a_{t'})P_M(s_t,a_t|s_{t'},a_{t'})\mathbb{I}(t'<t)$}\\
    \cline{1-1}
    Row$(s_{t'},a_{t'},r_{t'})$\quad\quad Column$(s_{t},a_{t},2)$ & \\
    \hline
    Row$(s_{t'},a_{t'},s_{t'+1})$\quad Column$(s_{t'},a_{t'},2)$ & $-P_M(s_{t'},a_{t'})$\\
    \hline
    Others & 0 \\
    \toprule[1.0pt]
    \end{tabular}
    \caption{The elements of matrix $X$}
    \label{tab:CR:elements_of_X}
\end{table}

\noindent where the column indexing tuple consists of one state and one action at step $t$ and one integer 1 or 2. $(s_t,a_t,1)$ will only multiply with the row of $F$ corresponding to the constraint condition of the state transition given $(s_t,a_t)$. Similarly, $(s_t,a_t,2)$ will only multiply with the row of $F$ corresponding to the constraint condition of the reward distribution given $(s_t,a_t)$.

\paragraph{Calculate the Lower Bound}
Let's use $\diag(\mathcal{B})$ to denote a block diagonal matrix consists of the matrices in set $\mathcal{B}$. After a similar discuss as \cite{jiang2016doubly}, with the proper choice of $U$, the value of matrix $U(U\trans I U)^{-1}U\trans$ should be:
\begin{equation}
U(U\trans I U)^{-1}U\trans={\diag}(\{B_s(s_t,a_t),B_r(s_t,a_t)\}_{t=-1}^{T}).
\end{equation}
and
\begin{align*}
B_s(s_t,a_t)=&\frac{\diag(P_s(\cdot|s_t,a_t))-P_s(\cdot|s_t,a_t)P_s(\cdot|s_t,a_t)\trans}{P_M(s_t,a_t)}. \\
B_r(s_t,a_t)=&\frac{\diag(P_r(\cdot|s_t,a_t))-P_r(\cdot|s_t,a_t)P_r(\cdot|s_t,a_t)\trans}{P_M(s_t,a_t)}.
\end{align*}
where $P_s(\cdot|s_t,a_t)$ and $P_r(\cdot|s_t,a_t)$ denote the transition and reward probability vector, respectively. 

Next, we divide the column vector $K$ into multiple continuous parts. Denote $\kappa^r_{(s_t,a_t,:)}$ as the vector fragment consists of the components of $K$, whose index is an state-action-reward tuple starting with $s_t,a_t$. Similarly, we use $\kappa^s_{(s_t,a_t,:)}$ to denote the vector fragment consists of the component of $K$, whose index is an state-action-state tuple starting with $s_t,a_t$. As a result,
\begin{equation}
KU(U\trans I U)^{-1}U\trans K\trans=\sum_{t=-1}^{T-1}\sum_{s_t,a_t} (\kappa^s_{(s_t,a_t,:)})\trans B_s(s_t,a_t)\kappa^s_{(s_t,a_t,:)}+\sum_{t=0}^T\sum_{s_t,a_t}(\kappa^r_{(s_t,a_t,:)})\trans B_r(s_t,a_t)\kappa^r_{(s_t,a_t,:)}.
\end{equation}
We first take a look at the reward part. Recall what we want to estimate is
\begin{equation}
\frac{\partial V}{\partial \theta_i}=P_{\pb}(\tau)\Big(\sum_{t_1=0}^T\frac{\partial \log\ptb{t_1}}{\partial \theta_i}\sum_{t_2=t_1}^T\gamma^{t_2}r_{t_2}\Big).
\end{equation}
Then the partial derivative of $P(r_t|s_t,a_t)$, i.e. $K_{(s_t,a_t,r_t)}$, should be
\begin{align*}
&\frac{1}{P(r_t|s_t,a_t)}\sum_{\tau}\mathbb{I}[(s_t,a_t,r_t)\in\tau]P_{\pb}(\tau)\Big(\sum_{t_1=0}^T\frac{\partial \log\ptb{t_1}}{\partial \theta_i}\sum_{t_2=t_1}^T\gamma^{t_2}r_{t_2}\Big)\\
=&\frac{\gamma^t r_t}{P(r_t|s_t,a_t)}\sum_{\tau}\mathbb{I}[(s_t,a_t,r_t)\in\tau]P_{\pb}(\tau)\Big(\sum_{t_1=0}^t\frac{\partial \log\ptb{t_1}}{\partial \theta_i}\Big)+C_r(s_t,a_t)\\
=&\underbrace{\gamma^t r_t\sum_{\tau_{[0:t]}}\mathbb{I}[(s_t,a_t)\in\tau_{[0:t]}]P_{\pb}(\tau_{[0:t]})\Big(\sum_{t_1=0}^t\frac{\partial \log\ptb{t_1}}{\partial \theta_i}\Big)}_{x_{(s_t,a_t,r_t)}}+C_r(s_t,a_t).\numberthis\label{kappa_r}
\end{align*}
where we use $C_r(s_t,a_t)$ to represent the part which is determined given $(s_t,a_t)$, and use $\mathbb{I}[\cdot]$ as indicator function. Then we have

\begin{align*}
&(\kappa^r_{(s_t,a_t,:)})\trans B_r(s_t,a_t)\kappa^r_{(s_t,a_t,:)}\\
=&\frac{1}{P_M(s_t,a_t)}\Big[\sum_{r_t}P(r_{t}|s_t,a_t)\Big(C_r(s_t,a_t)+x_{(s_t,a_t,r_{t})}\Big)^2-\Big(\sum_{r_t}P(r_{t}|s_t,a_t)\Big(C_r(s_t,a_t)+x_{(s_t,a_t,r_{t})}\Big)\Big)^2\Big]\\
=&\frac{1}{P_M(s_t,a_t)}\Big[\sum_{r_t}P(r_t|s_t,a_t)x^2_{(s_t,a_t,r_t)}-\Big(\sum_{r_t}P(r_t|s_t,a_t)x_{(s_t,a_t,r_t)}\Big)^2\Big]\\
&+\frac{2C_r(s_t,a_t)}{P_M(s_t,a_t)}\Big[\sum_{r_t}P(r_t|s_t,a_t)x_{(s_t,a_t,r_t)}-\Big(\sum_{r_t}P(r_t|s_t,a_t)\Big)\Big(\sum_{r_t}P(r_t|s_t,a_t)x_{(s_t,a_t,r_t)}\Big)\Big]\\
&+\frac{C_r^2(s_t,a_t)}{P_M(s_t,a_t)}\Big[\sum_{r_t}P(s_t,a_t)-\Big(\sum_{r_t}P(s_t,a_t)\Big)^2\Big]\\
=&\frac{\gamma^{2t}}{P_M(s_t,a_t)}\Big[\sum_{\tau_{[0:t]}}\mathbb{I}[(s_t,a_t)\in\tau_{[0:t]}]P_{\pb}(\tau_{[0:t]})\Big(\sum_{t_1=0}^t\frac{\partial \log\ptb{t_1}}{\partial \theta_i}\Big)\Big]^2\Big(\sum_{r_t}P(r_t|s_t,a_t)r_t^2-\Big(\sum_{r_t}P(r_t|s_t,a_t)r_t\Big)^2\Big)\\
=&\frac{\gamma^{2t}\mv_{r_t|s_t,a_t}[r_t]}{P_M(s_t,a_t)}\Big[\sum_{\tau_{[0:t]}}\mathbb{I}[(s_t,a_t)\in\tau_{[0:t]}]P_{\pb}(\tau_{[0:t]})\Big(\sum_{t_1=0}^t\frac{\partial \log\ptb{t_1}}{\partial \theta_i}\Big)\Big]^2.
\end{align*}
Therefore,

\begin{align*}
&\sum_{t=0}^T\sum_{s_t,a_t}(\kappa^r_{(s_t,a_t,:)})\trans B_r(s_t,a_t)\kappa^r_{(s_t,a_t,:)}\\
=&\me\bigg[\sum_{t=0}^T\gamma^{2t}\mv_{r_t|s_t,a_t}[r_t]\Big[\frac{1}{P_M(s_t,a_t)}\sum_{\tau_{[0:t]}}\mathbb{I}[(s_t,a_t)\in\tau_{[0:t]}]P_{\pb}(\tau_{[0:t]})\Big(\sum_{t_1=0}^t\frac{\partial \log\ptb{t_1}}{\partial \theta_i}\Big)\Big]^2\bigg].\numberthis\label{CR:reward}
\end{align*}
As for the states part, 
we first calculate $K_{(s_t,a_t,s_{t+1})}$, i.e. the partial derivative of $P(s_{t+1}|s_t,a_t)$:
\begingroup
\allowdisplaybreaks
\begin{align*}
&\frac{1}{P(s_{t+1}|s_t,a_t)}\sum_{\tau}\mathbb{I}[(s_t,a_t,s_{t+1})\in\tau]P_\pb(\tau)\Big(\sum_{t_1=0}^T\frac{\partial \log\ptb{t_1}}{\partial \theta_i}\sum_{t_2=t_1}^T\gamma^{t_2}r_{t_2}\Big)\\
=&\frac{1}{P(s_{t+1}|s_t,a_t)}\sum_{\tau}\mathbb{I}[(s_t,a_t,s_{t+1})\in\tau]P_\pb(\tau)\Big(\sum_{t_1=0}^{t}\frac{\partial \log\ptb{t_1}}{\partial \theta_i}\sum_{t_2=t_1}^t\gamma^{t_2}r_{t_2}\\
&\quad\quad\quad\quad+\sum_{t_1=0}^t\frac{\partial \log\ptb{t_1}}{\partial \theta_i}\sum_{t_2=t+1}^T\gamma^{t_2}r_{t_2}+\sum_{t_1=t+1}^T\frac{\partial \log\ptb{t_1}}{\partial \theta_i}\sum_{t_2=t_1}^T\gamma^{t_2}r_{t_2}\Big)\\
=&C_s(s_t,a_t)+\frac{1}{P(s_{t+1}|s_t,a_t)}\sum_{\tau}\mathbb{I}[(s_t,a_t,s_{t+1})\in\tau]P_\pb(\tau)\Big(\sum_{t_1=0}^{t}\frac{\partial \log\ptb{t_1}}{\partial \theta_i}\sum_{t_2=t+1}^T\gamma^{t_2}r_{t_2}\Big)\\
&+\frac{1}{P(s_{t+1}|s_t,a_t)}\sum_{\tau}\mathbb{I}[(s_t,a_t,s_{t+1})\in\tau]P_\pb(\tau)\Big(\sum_{t_1=t+1}^{T}\frac{\partial \log\ptb{t_1}}{\partial \theta_i}\sum_{t_2=t_1}^T\gamma^{t_2}r_{t_2}\Big)\\
=&C_s(s_t,a_t)+\underbrace{\gamma^{t+1}\sum_{\tau_{[0:t]}}\mathbb{I}[(s_t,a_t)\in\tau_{[0:t]}] P_\pb(\tau_{[0:t]})\Big(V^\pb(s_{t+1})\sum_{t_1=0}^{t}\frac{\partial \log\ptb{t_1}}{\partial \theta_i}+\gbt V^\pb(s_{t+1})\Big)}_{y_{(s_t,a_t,s_{t+1})}}.\numberthis\label{kappa_s}
\end{align*}
\endgroup
where we use $C_s(s_t,a_t)$ to represent the part which is determined given $(s_t,a_t)$. Then, we have,
\begingroup
\allowdisplaybreaks
\begin{align*}
&(\kappa^s_{(s_t,a_t,:)})\trans B_s(s_t,a_t)\kappa^s_{(s_t,a_t,:)}\\
=&\frac{1}{P_M(s_t,a_t)}\Big[\sum_{s_{t+1}}P(s_{t+1}|s_t,a_t)\Big(C_s(s_t,a_t)+y_{(s_t,a_t,s_{t+1})}\Big)^2-\Big(\sum_{s_{t+1}}P(s_{t+1}|s_t,a_t)\Big(C_s(s_t,a_t)+y_{(s_t,a_t,s_{t+1})}\Big)\Big)^2\Big]\\
=&\frac{1}{P_M(s_t,a_t)}\Big[\sum_{s_{t+1}}P(s_{t+1}|s_t,a_t)y^2_{(s_t,a_t,s_{t+1})}-\Big(\sum_{s_{t+1}}P(s_{t+1}|s_t,a_t)y_{(s_t,a_t,s_{t+1})}\Big)^2\Big]\\
&+\frac{2C_s(s_t,a_t)}{P_M(s_t,a_t)}\Big(\sum_{s_{t+1}}P(s_{t+1}|s_t,a_t)y_{(s_t,a_t,s_{t+1})}-\Big(\sum_{s_{t+1}}P(s_{t+1}|s_t,a_t)\Big)\Big(\sum_{s_{t+1}}P(s_{t+1}|s_t,a_t)y_{(s_t,a_t,s_{t+1})}\Big)\Big)\\
&+\frac{C_s^2(s_t,a_t)}{P_M(s_t,a_t)}\Big[\sum_{s_{t+1}}P(s_{t+1}|s_t,a_t)-\Big(\sum_{s_{t+1}}P(s_{t+1}|s_t,a_t)\Big)^2\Big]\\
=&\frac{1}{P_M(s_t,a_t)}\Big[\sum_{s_{t+1}}P(s_{t+1}|s_t,a_t)y^2_{(s_t,a_t,s_{t+1})}-\Big(\sum_{s_{t+1}}P(s_{t+1}|s_t,a_t)y_{(s_t,a_t,s_{t+1})}\Big)^2\Big]\\
=&\frac{1}{P_M(s_t,a_t)}\mv_{s_{t+1}|s_t,a_t}[y_{(s_t,a_t,s_{t+1})}].
\end{align*}
\endgroup
Therefore,
\begin{align*}
\sum_{t=-1}^{T-1}\sum_{s_t,a_t}(\kappa^s_{(s_t,a_t,:)})\trans B_r(s_t,a_t)\kappa^s_{(s_t,a_t,:)}=&\me\Big[\sum_{t=-1}^{T-1}\mv_{s_{t+1}|s_t,a_t}\Big[\frac{y_{(s_{t},a_t,s_{t+1})}}{P_M(s_{t},a_{t})}\Big]\Big]\\
=&\me\Big[\sum_{t=0}^{T}\mv_{s_t|s_{t-1},a_{t-1}}\Big[\frac{y_{(s_{t-1},a_{t-1},s_{t})}}{P_M(s_{t-1},a_{t-1})}\Big]\Big].\numberthis\label{CR:transition}
\end{align*}
Combine (\ref{CR:reward}) and (\ref{CR:transition}) can we obtain the Cramer-Rao lower bound for DAG-MDP:
\begin{align*}
&\me\Big[\sum_{t=0}^T\gamma^{2t}\bigg\{\mv_{r_t|s_t,a_t}[r_t]\Big[\sum_{\tau_{[0:t]}}\mathbb{I}[(s_t,a_t)\in\tau_{[0:t]}]\frac{P_{\pb}(\tau_{[0:t]})}{P_M(s_t,a_t)}\Big(\sum_{t_1=0}^t\frac{\partial \log\ptb{t_1}}{\partial \theta_i}\Big)\Big]^2\\
&\quad\quad\quad\quad+\mv_{s_t|s_{t-1},a_{t-1}}\Big[\sum_{\tau_{[0:t-1]}}\mathbb{I}[(s_{t-1},a_{t-1})\in\tau_{[0:t-1]}]\frac{P_\pb(\tau_{[0:t-1]})}{P_M(s_{t-1},a_{t-1})}\Big(V_t^\pb\sum_{t_1=0}^{t-1}\frac{\partial \log\ptb{t_1}}{\partial \theta_i}+\frac{\partial V_t^\pb}{\partial \theta_i}\Big)\Big]\bigg\}\Big].
\end{align*}
\end{proof}

\begin{remark} \label{rem:achieve}
As for Tree MDP, which is the special case of DAG-MDP, for each state $s_t$, there only exists one trajectory starting from step 0 and end at $s_t$. Therefore, $\mv_{s_t|s_{t-1},a_{t-1}}=\mv_{s_t|s_0,a_0,...,s_{t-1},a_{t-1}}$, and we use $\mv_t$ to replace $\mv_{s_t|s_{t-1},a_{t-1}}$. As a result, the lower bound for Tree-MDP case should be

\begin{align*}
&\me\Big[\sum_{t=0}^T\gamma^{2t}\bigg\{\mv_{r_t|s_t,a_t}[r_t]\Big[\sum_{\tau_{[0:t]}}\mathbb{I}[(s_t,a_t)\in\tau_{[0:t]}]\frac{P_{\pb}(\tau_{[0:t]})}{P_M(s_t,a_t)}\Big(\sum_{t_1=0}^t\frac{\partial \log\ptb{t_1}}{\partial \theta_i}\Big)\Big]^2\\
&\quad\quad\quad\quad+\mv_{s_t|s_{t-1},a_{t-1}}\Big[\sum_{\tau_{[0:t-1]}}\mathbb{I}[(s_{t-1},a_{t-1})\in\tau_{[0:t-1]}]\frac{P_\pb(\tau_{[0:t-1]})}{P_M(s_{t-1},a_{t-1})}\Big(V_t^\pb\sum_{t_1=0}^{t-1}\frac{\partial \log\ptb{t_1}}{\partial \theta_i}+\frac{\partial V_t^\pb}{\partial \theta_i}\Big)\Big]\bigg\}\Big]\\
=&\me\Big[\sum_{t=0}^T\gamma^{2t}\bigg\{\mv_{t+1}[r_t]\Big[\sum_{r_t}P(r_t|s_t,a_t)\Big(\sum_{t_1=0}^t\frac{\partial \log\ptb{t_1}}{\partial \theta_i}\Big)\Big]^2\\
&+\mv_{t}\Big[\sum_{r_{t-1}}P(r_{t-1}|s_{t-1},a_{t-1})\Big(V_t^\pb\sum_{t_1=0}^{t-1}\frac{\partial \log\ptb{t_1}}{\partial \theta_i}+\frac{\partial V_t^\pb}{\partial \theta_i}\Big)\Big]\bigg\}\Big]\\
=&\me\Big[\sum_{t=0}^T\gamma^{2t}\bigg\{\mv_{t+1}[r_t]\Big[\Big(\sum_{t_1=0}^t\frac{\partial \log\ptb{t_1}}{\partial \theta_i}\Big)\Big]^2+\mv_t\Big[\Big(V_t^\pb\sum_{t_1=0}^{t-1}\frac{\partial \log\ptb{t_1}}{\partial \theta_i}+\frac{\partial V_t^\pb}{\partial \theta_i}\Big)\Big]\bigg\}\Big].
\end{align*}
We can observe that, if both $Q^\pb$ and $\gbt Q^\pb$ are well-estimated by $\hq^\pb$ and $\gbt\hq^\pb$, the variance of the PG estimator derived from doubly robust OPE estimator, i.e. Eq.\eqref{drpg:maintext:variance_all}, can achieve this lower bound.
\end{remark}

\section{Experiment Details}\label{exp:true_estimators}
We provide detailed specifications of the methods compared in the experiments, which are (a) Per-step IS, (b) Per-step IS with state-dependent baseline,  (c) Per-step IS with state-action-dependent baseline, (d) Per-step IS with trajectory-wise control variate, and (e) our DR-PG estimator. As for (a)--(d), we use the same estimators and hyperparameter setting as \citet{cheng2019trajectory}. Besides, we implement our DR-PG estimator based on their implementation of trajectory-wise control variate, and inherit their practical modifications as-is. The detailed formula for these estimators are given below.
\paragraph{(a) Per-step IS}
\begin{equation}
\sum_{t=0}^T\gbt\log\ptb{}(a_t|s_t)\Big[\sum_{t'=t}^T(\gamma\delta)^{t'-t}r_{t'}\Big].\label{exp:mc}
\end{equation}

\paragraph{(b) Per-step IS with State-Dependent Baseline}
\begin{equation}
\sum_{t=0}^T\gbt\log\ptb{}(a_t|s_t)\Big[\sum_{t'=t}^T(\gamma\delta)^{t'-t}r_{t'} - \tilde{V}(s_{t'})\Big].\label{exp:st}
\end{equation}

\paragraph{(c) Per-step IS with State-Action-Dependent Baseline}
\begin{equation}
\sum_{t=0}^T\Big\{\gbt\log\pi(a_t|s_t)\Big[\sum_{t'=t}^T(\gamma\delta)^{t'-t}r_{t'}\Big]-\Big(\tq(s_{t},a_{t})\gbt\log\pi(a_t|s_t)-\tilde{G}_1(s_{t}) \Big)\Big\}.\label{exp:sa}
\end{equation}

\paragraph{(d) Per-step IS with Trajectory-wise Control Variate}
\begin{align*}
\sum_{t=0}^T\Big\{&\gbt\log\pi(a_t|s_t)\Big[\sum_{t_1=t}^T(\gamma\delta)^{t_1-t}r_{t_1}-\sum_{t_2=t+1}^T(\theta\delta)^{t_2-t} \Big(\tq(s_{t'},a_{t'})-\tv(s_{t'}) \Big)\Big]\\
&-\Big(\tq(s_{t'},a_{t'})\gbt\log\pi(a_t|s_t)-\tilde{G}_1(s_{t'})\Big)\Big\}.\numberthis\label{exp:traj}
\end{align*}
Note that the $\theta$ and $\delta$ hyperparameters are the practical modifications introduced by \citet{cheng2019trajectory}, and we inherent them as-is in both (d) and (e).

\paragraph{(e) DR-PG}

\begin{align*}
\sum_{t=0}^T\Big\{&\gbt\log\pi(a_t|s_t)\Big[\sum_{t_1=t}^T(\gamma\delta)^{t_1-t}r_{t_1}-\sum_{t_2=t+1}^T(\theta\delta)^{t_2-t} \Big(\tq(s_{t'},a_{t'})-\tv(s_{t'}) \Big)\Big]\\
&-\Big(\tq(s_{t'},a_{t'})\gbt\log\pi(a_t|s_t)-\tilde{G}_1(s_{t'})\Big)-\Big(\tilde{\gbt Q}(s_{t'},a_{t'})-\tilde{G}_2(s_{t'})\Big)\Big\}.\numberthis\label{exp:dr}
\end{align*}
Next we explain some new notations we used in the above equations. First, $\tilde{V}$ in (b) is the value function estimator. The functions $\tq,\tv$ and $\tilde{G}_1$ in (c)--(e) is defined as
\begin{align*}
\tq(s_t) =& r_t + \delta \tilde{V}(s_{t+1}).\\
\tilde{V}(s_t) =& \frac{1}{n}\sum_{i=1}^n \Big(r(s_t, a_t^i)+\delta \tilde{V}(s^i_{t+1})\Big).\\
\tilde{G}_1(s_t)=&\frac{1}{n}\sum_{i=1}^n\Big(r(s_t, a_t^i)+\delta \tilde{V}(s^i_{t+1})-\tilde{V}(s_t)\Big)\gbt \log\ptb{}(a_t|s_t).
\end{align*}
where $a_t^i\sim \ptb{}(\cdot|s_t)$, and both $r(s_t,a^i)$ and $s^i_{t+1}$ are sampled from $\tilde{d}$ given $(s_t,a^i_t)$, and $\tilde{V}(s_t)$ in $\title{G}_1(s_t)$ is used for reducing estimation variance. Besides, in (\ref{exp:dr}), we have two functions $\tilde{\gbt Q}$ and $\tilde{G}_2$ defined as
\begin{align*}
\tilde{\gbt Q}(s,a)=&\frac{1}{n_q}\sum_{i=1}^{n_q}\sum_{k=0}^{L-1}(\gamma')^k\Big(r(s^i_k,a^i_k)+\delta \tilde{V}(s^i_{k+1})- \tilde{V}(s^i_k)\Big)\gbt\log\ptb{}(a^i_k|s^i_k).\\
\tilde{G}_2(s)=&\frac{1}{n_v}\sum_{i=1}^{n_v}\tilde{\gbt Q}(s,a^i).
\end{align*}
where $\{s^i_k,a^i_k,r(s^i_k,a^i_k)\}_{k=0}^{L-1}$ is a trajectory sampled from $\tilde{d}$ starting from $(s,a)$, and $a^i\sim \pi(\cdot|s)$.

In the experiments, we use $n=1000,n_q=20,n_v=20,L=30, \gamma=1,\delta=0.999,\theta=\gamma'=0.9$.

\end{document}